\documentclass{article}

\usepackage{microtype}
\usepackage{multicol}
\usepackage{graphicx}
\usepackage{epstopdf}
\usepackage{epsfig}
\usepackage{subfigure}
\usepackage{booktabs} 
\usepackage{amssymb,amsfonts,amsmath,amsthm}

\newtheorem{assumption}{Assumption}
\newtheorem{theorem}{Theorem}
\newtheorem{lemma}{Lemma}

\usepackage{hyperref}

\usepackage[accepted]{icml2020}

\icmltitlerunning{Distributed Optimization over Block-Cyclic Data}

\usepackage{xpatch}
\makeatletter

\xpatchcmd{\algorithmic}
  {\newcommand{\STATE}{\ALC@it}}
  {\newcommand{\STATE}{\@ifstar\STATEstar\STATEnostar}}
  {}{}
\newcommand{\STATEstar}{\item[]}
\newcommand{\STATEnostar}{\ALC@it}

\begin{document}

\twocolumn[
\icmltitle{Distributed Optimization over Block-Cyclic Data}



\icmlsetsymbol{equal}{*}

\begin{icmlauthorlist}
\icmlauthor{Yucheng Ding}{sjtu}
\icmlauthor{Chaoyue Niu}{sjtu}
\icmlauthor{Yikai Yan}{sjtu}
\icmlauthor{Zhenzhe Zheng}{sjtu}
\icmlauthor{Fan Wu}{sjtu}
\icmlauthor{Guihai Chen}{sjtu}
\icmlauthor{Shaojie Tang}{texas}
\icmlauthor{Rongfei Jia}{ali}
\end{icmlauthorlist}

\icmlaffiliation{sjtu}{Shanghai Jiao Tong University, China}
\icmlaffiliation{texas}{University of Texas at Dallas, USA}
\icmlaffiliation{ali}{ Alibaba Group, Beijing, China}

\icmlcorrespondingauthor{Fan Wu}{fwu@cs.sjtu.edu.cn}

\icmlkeywords{Machine Learning, ICML}

\vskip 0.3in
]



\printAffiliationsAndNotice{}  

\begin{abstract}
We consider practical data characteristics underlying federated learning, where unbalanced and non-i.i.d. data from clients have a block-cyclic structure: each cycle contains several blocks, and each client's training data follow block-specific and non-i.i.d. distributions. Such a data structure would introduce client and block biases during the collaborative training: the single global model would be biased towards the client or block specific data. To overcome the biases, we propose two new distributed optimization algorithms called multi-model parallel SGD (MM-PSGD) and multi-chain parallel SGD (MC-PSGD) with a convergence rate of $O(1/\sqrt{NT})$, achieving a linear speedup with respect to the total number of clients. In particular, MM-PSGD adopts the block-mixed training strategy, while MC-PSGD further adds the block-separate training strategy. Both algorithms create a specific predictor for each block by averaging and comparing the historical global models generated in this block from different cycles. We extensively evaluate our algorithms over the CIFAR-10 dataset. Evaluation results demonstrate that our algorithms significantly outperform the conventional federated averaging algorithm in terms of test accuracy, and also preserve robustness for the variance of critical parameters.
\end{abstract}

\section{Introduction}\label{section:intro}

Federated learning (FL) allows multiple clients to collaborate in the training of a global machine learning model under the coordination of a cloud server without sharing raw data~\cite{DBLP:conf/aistats/McMahan17}. In this setting, the clients (e.g., millions of mobile-device users or hundreds of companies and organizations) train the model in {parallel} using their local data, and the cloud server updates the global model by aggregating the local models collected from the clients in communication iterations.

As a new paradigm of distributed machine learning, the data characteristics in FL significantly differ from those in the traditional distributed optimization~\cite{MuOsdi2014Distributed,LianIcml2018Adpsgd, HanlinNips2018Compress, HanlinIcml2018D2, HanlinIcml2019DS, HaoIcml2019Dynamic}. On one hand, considering the fact that the clients tend to have diverse usage patterns, the amount of local data across clients are usually different, and each client's local dataset just represents a certain aspect of the overall data distribution. That is, the data distributions in FL are \emph{unbalanced} and \emph{non-i.i.d.}~\cite{KevinCorr2019Noniid, MehryarIcml2019Agnos, DBLP:journals/corr/abs-1908-07873, survey2}.
On the other hand, FL data are usually \emph{periodically variational}, which comes from several practical reasons. Due to strict data protection regulations~\cite{link:gdpr, proc:nips19:data:deletion} and resource constraints, in many cases, the clients cannot hold user data for a long time. Thus, the training data may change cyclically over time and follow a certain temporal pattern. For the worldwide FL applications, the training data is also periodically variational since the available clients often follow a diurnal pattern~\cite{DBLP:journals/corr/abs-1908-07873}. We can use a concept of block-cyclicity to model the periodical variation of training data in FL. The training process is composed of several cycles, each of which further contains several data blocks, representing different training data distributions through the cycle. Within each data block, the clients with unbalanced and non-i.i.d. data jointly train the global model in a parallel way.

There has been some effort on developing convergence guarantees for FL algorithms, but none of the existing work has considered the practical data characteristics, i.e., unbalance, non-i.i.d. distribution, and block-cyclic pattern. Some aspects of data features have been partially investigated in the literature. \citet{DBLP:conf/aaai/Yu19} provided a theoretical analysis for the federated averaging (FedAvg), also known as parallel restarted SGD, by assuming that data distributions are non-i.i.d. but remain unchanged through the training process; \citet{DBLP:conf/icml/Eichner19} considered the block-cyclic data pattern, but simply assumed there is only one client.

The above discussed data characteristics introduce two major biases in FL: (1) {\bf client bias:} the model trained on a client would be biased to the client's local training data; and (2) {\bf block bias:} the model trained using the data from a block would be skewed towards the data distribution in this block. To migrate the client bias, we aggregate the local model updates from participating clients, and to overcome the block bias, instead of training a single global model for all the blocks, we construct a series of block-specific predictors\footnote{Throughout this paper, we use the block-specific global model and predictor interchangeably.} by aggregating the model updates from the corresponding block in different cycles. Based on the above basic ideas, we first propose \textbf{Multi-Model Parallel SGD} (MM-PSGD), which takes a block-mixed training strategy, i.e., the training process goes through the mixture of different blocks, but for each block, we average the historical global models generated over it to construct the specific predictor.
For the federated optimization with a strongly convex and smooth objective, MM-PSGD always converges to the optimal global model at a rate of $O(1/\sqrt{NT})$, achieving a linear speedup in terms of the number of clients. MM-PSGD obtains good performance when the block-specific data distributions are not far away from each other. To further improve the performance for the case of extremely different data distributions across blocks, we propose {\bf Multi-Chain Parallel SGD} (MC-PSGD), which augments MM-PSGD with a block-separate training strategy. With this strategy, we construct a set of block-separate global models using only the training data from the corresponding block. The critical step behind MC-PSGD is that in each training round, we select the ``better'' block-specific global models generated from the block-mixed training process and the block-separate training process. We show that MC-PSGD further ensures each block-specific predictor to converge to the block's optimal model at a rate of $O(\sqrt{M}/\sqrt{NT})$ while adding a slight communication overhead.

Our key contributions in this work can be summarized as follows: (1) To the best of our knowledge, we are the first to consider that the data distributions in FL are unbalanced, non-i.i.d., and block-cyclic. (2) Under the practical data characteristics, we propose MM-PSGD and MC-PSGD, both of which return a set of block-specific predictors and have a convergence guarantee of $O(1/\sqrt{NT})$ with respective to the optimal global model. MC-PSGD further ensures that each block-specific predictor would converge to the block's optimal model. (3) We evaluate our algorithms over the CIFAR-10 dataset. Evaluation results demonstrate that our algorithms have significant performance improvement: achieving 6\% higher test accuracy compared with FedAvg, and preserve robustness for the variance of critical parameters, whereas FedAvg fluctuates intensely due to the block-cyclic pattern in training data.

\section{Preliminaries}
We consider a general distributed optimization problem:
 \begin{align}
\min \limits_{\mathbf{x}} F\left(\mathbf{x}\right) = \sum_{i=1}^N p_i\tilde{F}^i\left(\mathbf{x}\right),\label{eq:start}
\end{align}
where $N$ is the number of clients, and $p_i$ is the weight of the $i$-th client such that $p_i\ge 0$ and $\sum_{i=1}^N p_i=1$. Each function $\tilde{F}^i\left(\mathbf{x}\right)$ is defined by:
\begin{align}
\tilde{F}^i\left(\mathbf{x}\right)\overset{\triangle}{=}\mathbb{E}_{\xi \sim D^i}\left[f\left(\mathbf{x},\xi\right)\right],\label{eq:localfunc}
\end{align}
where $D^i$ is the overall distribution of client $i$'s local data, and $f\left(\mathbf{x}, \xi\right)$ is a loss function on the data $\xi$ from $D^i$. To simplify the weighted form, we further let $F^i\left(\mathbf{x}\right)$ represent $p_iN\tilde{F}^i\left(\mathbf{x}\right)$ to scale the local objective. Then, the global objective $F\left(\mathbf{x}\right)$ becomes an average of $F^i\left(\mathbf{x}\right)$:
\begin{align}
 F\left(\mathbf{x}\right)=\frac{1}{N} \sum_{i=1}^N F^i\left(\mathbf{x}\right).\label{eq:start1}
\end{align}

We next formalize the practical cyclicity of training data. There are $C$ cycles in total, each of which consists of $M$ different global data blocks. In FL, the cycles can be the days of model training, and the global data blocks correspond to daytime and nighttime in each day. The $m$-th global data block $D_m$ is further comprised of $N$ local data blocks. Formally, we have
\begin{align}
D_m=\frac{1}{N}\sum_{i=1}^N D_m^i,\label{eq:def-block}
\end{align}
where $D_m^i$ denotes the distribution of client $i$'s $m$-th local data block. Within a block, there will be $E$ rounds as well as $K = E\times I$ iterations in total, where $I$ denotes the number of local iterations in each round. Under such a data-cyclic model, for each client $i$, its data samples are drawn from the distribution
\begin{equation}
\begin{aligned}
&\xi_{t\left(c,m,k\right)}^i\sim D_m^i,\\
&t\left(c,m,k\right) = \left(c–1\right)MK + \left(m–1\right)K+ k,
\end{aligned}\label{eq:sample}
\end{equation}
where $c \in \left\{1,2,\ldots,C\right\}$ indexes the cycles, $m \in \left\{1,2,\ldots,M\right\}$ indexes the blocks, and $k \in \left\{1,2,\ldots,K\right\}$ indexes the local iterations within a block.

Now, the global data distribution is actually \textbf{block-cyclic}, and we rewrite the original global optimization objection in a block-cyclic way:
\begin{align}
F\left(\mathbf{x}\right)=\frac{1}{M}\sum_{m=1}^M F_m\left(\mathbf{x}\right),\label{eq:component-func}
\end{align}
where the block-specific function $F_m\left(\mathbf{x}\right)$ is an average of $F_m^i\left(\mathbf{x}\right)$ in the corresponding block, and $F_m^i\left(\mathbf{x}\right)\overset{\triangle}{=}\mathbb{E}_{\xi \sim D_m^i}\left[f\left(\mathbf{x},\xi \right)\right]$. In this work, we make the following assumptions on the function $F_m^i(\mathbf{x})$.

\begin{assumption}[Strong Convexity]\label{assum:no1}
$F_m^i\left(\mathbf{x}\right)$ is strongly convex with modulus $\mu$: for any $\mathbf{x}, \mathbf{y}$,
\begin{align*}
F_m^i\left(\mathbf{y}\right) \ge F_m^i\left(\mathbf{x}\right) + \left<\mathbf{y} - \mathbf{x}, \nabla F_m^i\left(\mathbf{x}\right)\right> + \frac{\mu}{2}\parallel \mathbf{x}-\mathbf{y}\parallel^2.
\end{align*}
\end{assumption}

\begin{assumption}[Smoothness]\label{assum:no2}
$F_m^i\left(\mathbf{x}\right)$ is smooth with modulus $L$: for any $\mathbf{x},\mathbf{ y}$,
\begin{align*}
F_m^i\left(\mathbf{y}\right)\le F_m^i\left(\mathbf{x}\right)+\left<\mathbf{y}-\mathbf{x},\nabla F_m^i\left(\mathbf{x}\right)\right>+\frac{L}{2}\parallel \mathbf{x}-\mathbf{y}\parallel^2.
\end{align*}
\end{assumption}

To establish the convergence results, we further make some assumptions about the local gradients and the feasible space of model parameters.
\begin{assumption}[Bounded Variance]\label{assum:no3}
During local training, the variance of stochastic gradients on each client is bounded by $\sigma^2$:
\begin{align*}
&\forall \mathbf{x},m,i:\ \mathbb{E}_{\xi\sim D_m^i}\left[\parallel\nabla F_m^i\left(\mathbf{x}\right) - \nabla f\left(\mathbf{x},\xi\right)\parallel^2\right] \le \sigma ^2.
\end{align*}
\end{assumption}

\begin{assumption}[Bounded Gradient Norm]\label{assum:no4}
The expected $l_2$-norm of the stochastic gradients is bounded by $G^2$:
\begin{align*}
&\forall \mathbf{x},m,i:\ \mathbb{E}_{\xi\sim D_m^i}\left[ \parallel\nabla f\left(\mathbf{x},\xi\right)\parallel^2\right] \le G^2.
\end{align*}
\end{assumption}

\begin{assumption}[Bounded Model Parameters]\label{assum:no5}
The $l_2$-norm of any model parameters is bounded by $B^2$:
\begin{align*}
\forall \mathbf{x}:\ \parallel \mathbf{x} \parallel^2 \le B^2.
\end{align*}
\end{assumption}

The above assumptions have also been made in the literature to derive convergence results. Assumptions \ref{assum:no1} and \ref{assum:no2} are standard. Assumptions \ref{assum:no3} and \ref{assum:no4} were made in \cite{DBLP:conf/nips/StichCJ18, DBLP:conf/iclr/Stich19, DBLP:conf/icml/Yu19, DBLP:conf/aaai/Yu19}. Assumption \ref{assum:no5} was made in \cite{DBLP:conf/icml/Zinkevich03, DBLP:conf/icml/Eichner19}.

Some related works have investigated the convergence of FL algorithms with either the non-i.i.d. data distribution or the cyclic data characteristics. To the best of our knowledge, none of existing work has jointly considered these two characteristics. 
\citet{DBLP:conf/aaai/Yu19} proved that FedAvg achieves an $O(1/\sqrt{NT})$ convergence for a non-convex objective under the assumptions that the data distributions are non-i.i.d. but remains unchanged during the training process. This is actually a special case of our problem by setting $M=1$, indicating there is only one data block. \citet{DBLP:conf/icml/Eichner19} observed that block-cyclic data characteristics can deteriorate the performance of both sequential SGD and parallel SGD. However, they only proposed an approach for a sequential case with $O(1/\sqrt{T})$ convergence guarantee. This is another special case of our problem by setting $N = 1$, indicating that there is only one client in total.

In what follows, we discuss how to design FL algorithms to support arbitrary $M$ and $N$ while attaining an $O(1/\sqrt{NT})$ convergence guarantee. The convergence rate is independent on $M$ and is consistent to the result without considering cyclic data feature in \citet{DBLP:conf/aaai/Yu19}. In addition, by setting $N=1$, the convergence result would reduce to $O(1/\sqrt{T})$ as in \citet{DBLP:conf/icml/Eichner19} without considering the non-i.i.d. data distribution.

\begin{figure}[!t]
\centering
\subfigure[An exemplary workflow of MM-PSGD]{
\centering
\includegraphics[width=0.95\columnwidth]{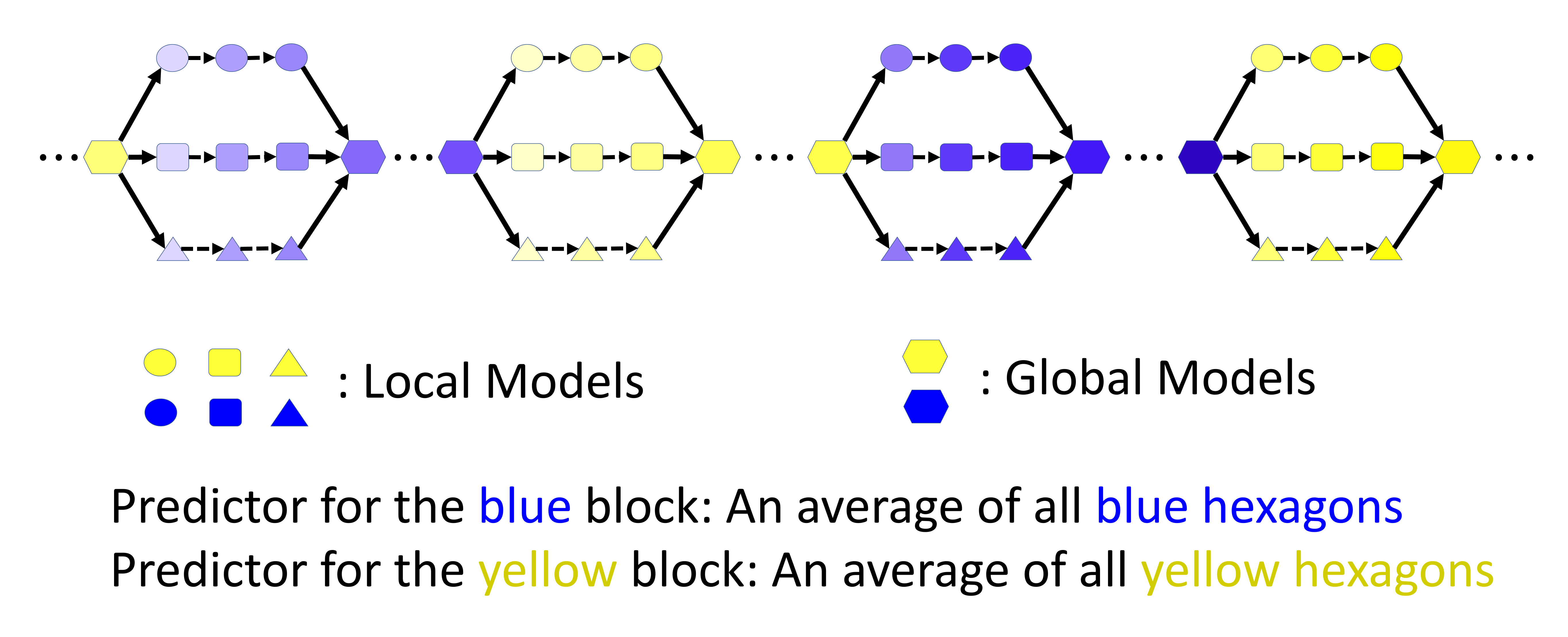}
\label{fig-alg1}
}
\subfigure[An exemplary workflow of MC-PSGD]{
\centering
\includegraphics[width=0.95\columnwidth]{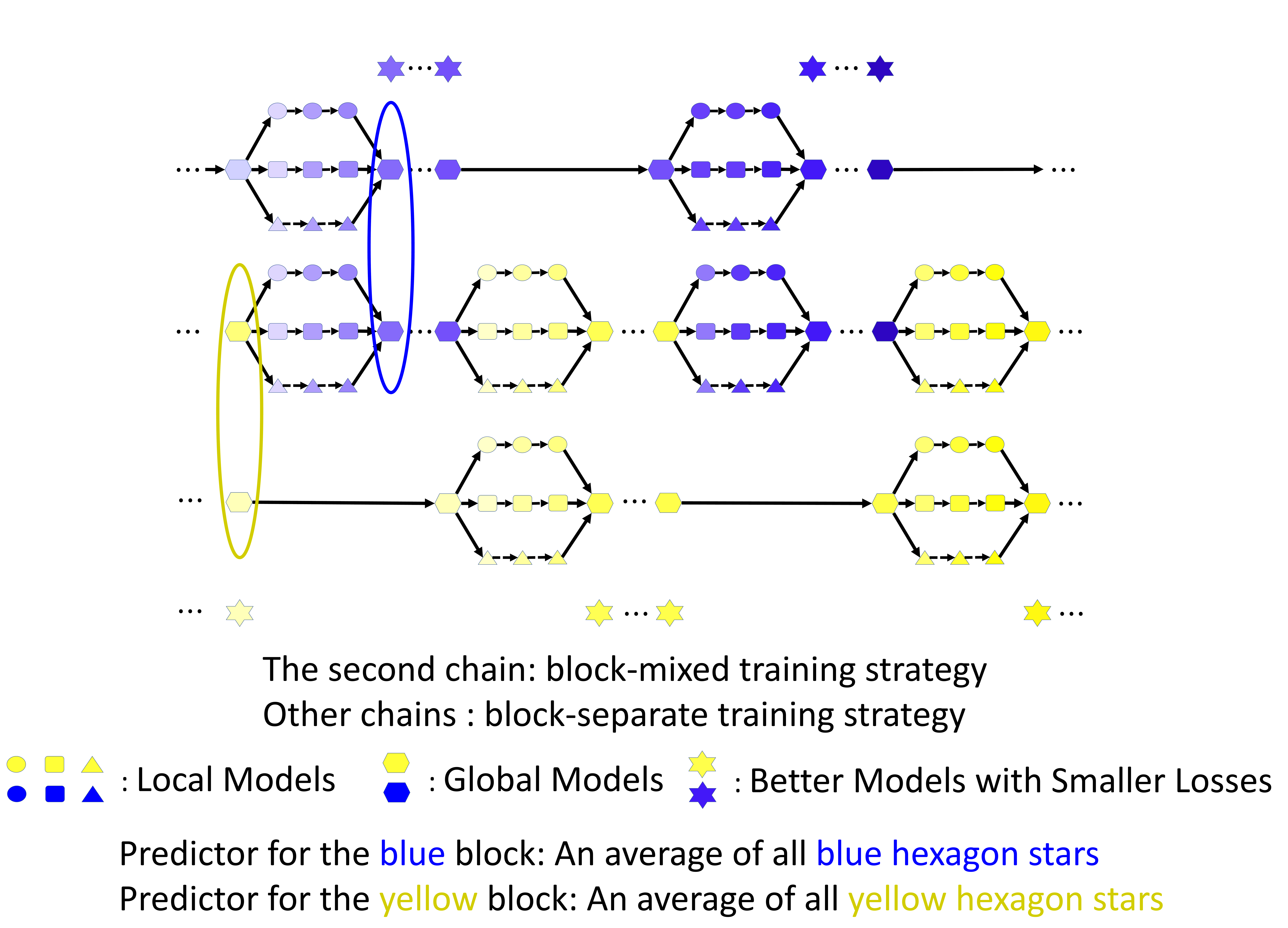}
\label{fig-alg2}
}
\centering
\caption{Exemplary workflows of MM-PSGD and MC-PSGD with $M=2$ blocks and $N=3$ clients. Different colors denote different blocks, and different shapes in a certain color denote parallel clients with a certain block.
}\label{fig:2algs}
\end{figure}

\section{Algorithm Design}
In this section, we propose two algorithms to construct a set of block-specific predictors, making a trade-off between performance guarantee and communication efficiency. To guarantee a mini-max optimal error with respective to the single optimal global model\footnote{Throughout this paper, we use mini-max optimal error to denote the difference between the average loss of our predictors over $M$ blocks and the loss of the single optimal global model.} and reduce communication overhead, we propose Multi-Model Parallel SGD, namely {MM-PSGD}. To further ensure each predictor to converge to the block's optimal model with a slight additional communication cost, we propose Multi-Chain Parallel SGD, namely {MC-PSGD}. Figure~\ref{fig:2algs} illustrates the workflows of MM-PSGD and MC-PSGD.

\subsection{Multi-Model Parallel SGD}\label{subsec:3.1}
In this subsection, we design MM-PSGD, which outputs a set of block-specific predictors with a mini-max optimal error guarantee with respective to the single optimal global model.
Considering that data distributions are non-i.i.d. and block-cyclic, there exist client biase and block biase as introduced in Section \ref{section:intro}.
To overcome the biases, we execute the training process over the mixture of blocks, but for each block, we average the historical global models generated over it to obtain the corresponding predictor.

\begin{algorithm}[!t]
   \caption{Multi-Model Parallel SGD (MM-PSGD)}\label{alg:alg1}
\begin{algorithmic}[1]
   \STATE {\bfseries Input:} Learning rate $\gamma$; Number of blocks $M$; Number of total iterations $T$; Number of local iterations in each communication round $I$.
   \STATE Initialize predictors ${\bf X}$ and local models;
   \FOR{$t=1$ {\bfseries to} $T$}
  		 \IF{$t$ is a multiple of $I$}
		  \STATE* {\tt /* Cloud server's process */}
 		  \STATE Updates the global model: $\bar{\mathbf{x}}_t\gets\frac{1}{N}\sum_{i=1}^N\mathbf{x}_{t}^i$;
  		 \STATE $m \leftarrow$ the current block index;
  		 \STATE $r \leftarrow$ the existing rounds in block $m$;
   		\STATE $\tilde{\mathbf{x}}_m \gets \frac{r}{r+1}\tilde{\mathbf{x}}_{m} + \frac{1}{r+1}\bar{\mathbf{x}}_t$;
		\STATE Broadcasts $\bar{\mathbf{x}}_t$ to all clients: $\mathbf{x}_t^i \gets \bar{\mathbf{x}}_t$; \label{alg1:broad}
		\ENDIF

		\STATE* {\tt /* Each client $i$'s process */}
		\STATE Observes a stochastic gradient $\mathbf{g}_t^i$; \label{alg1:ob}
		\STATE $\mathbf{x}_{t+1}^i\gets\mathbf{x}_t^i-\gamma \mathbf{g}_t^i$; \label{alg1:localup}
   \ENDFOR
   \STATE {\bfseries return} ${\bf X}$
\end{algorithmic}
\end{algorithm}

We sketch MM-PSGD in Algorithm \ref{alg:alg1}. At the beginning, we initialize the learning rate $\gamma$, the number of local iterations $I$, and the vector of predictors $\mathbf{X}=(\tilde{\mathbf{x}}_1, \tilde{\mathbf{x}}_2, \cdots, \tilde{\mathbf{x}}_M)$. The vector $\mathbf{X}$ records the latest predictor for each block. In the training process, there are $T$ iterations and $T/I$ rounds in total.
If $t$ is a multiple of $I$, then the iteration $t$ is a \emph{communication iteration}.
At each communication iteration, the cloud server collects and aggregates the local models from all participating clients to obtain the new global model $\bar{\bf x}_t$ (Line 5), updates the block-specific predictor $\tilde{\mathbf{x}}_m$ for the current block (Line 8), and pushes the new global model to all clients (Line 9).
After receiving the latest global model, each client runs $I$ local SGD iterations according to the observed local gradients in parallel (Lines 11 to 12).

After $T$ iterations in total, the algorithm will return $M$ block-specific predictors ${\mathbf{X}}$. According to Line 8, we can verify that the final predictor $\tilde{\mathbf{x}}_m$ for the block $m$ is the
average of the historical global models calculated at communication iterations belonging to the block $m$, i.e.,
\begin{align}
\tilde{\mathbf{x}}_m=\frac{1}{\parallel \mathcal{R}^m \parallel}\sum_{t\in \mathcal{R}^m}\bar{\mathbf{x}}_t,\label{eq:finalmodel-1}
\end{align}
where $\mathcal{R}^m$ is the set of communication iterations corresponding to the block $m$.

\subsection{Multi-Chain Parallel SGD}\label{subsec:3.2}

\begin{algorithm}[!t]
   \caption{Multi-Chain Parallel SGD (MC-PSGD)}\label{alg:alg2}
   \label{alg:example}
	\begin{algorithmic}[1]
   	\STATE {\bfseries Input:} Learning rates $\gamma$ and $\eta$; Number of blocks $M$; Number of total iterations $T$; Number of local iterations in each communication round $I$.
	\STATE Initialize ${\bf U}$, ${\bf X}$, and ${\bf Y}$;
  	 \FOR{$t=1$ {\bfseries to} $T$}
  		 \IF{$t$ is a multiple of $I$}
			\STATE* {\tt /* Cloud server's process */}
			\STATE $m \leftarrow $ the current block index;
            \STATE Updates block-mixed model: $\bar{\mathbf{x}}_t\gets\frac{1}{N}\sum_{i=1}^N \mathbf{x}_{t}^i$;
            \STATE Updates block-separate model: \\
            $\bar{\mathbf{y}}_{t}\gets \frac{1}{N}\sum_{i=1}^N \mathbf{y}_{t}^i$; \ \ \ $\bar{\mathbf{w}}_{m} \gets \bar{\mathbf{y}}_{t}$;
			\STATE Broadcasts $\bar{\mathbf{x}}_t$ and $\bar{\mathbf{y}}_{t}$ to all clients; \label{bro}
     			\STATE Receives local losses of $\bar{\mathbf{x}}_t$ and $\bar{\mathbf{y}}_{t}$ from all clients;
			\STATE$\bar{\mathbf{u}}_t\leftarrow$ $\bar{\mathbf{x}}_t$ or $\bar{\mathbf{y}}_{t}$ with a smaller average local loss;
   			\STATE $r \leftarrow$ the existing rounds in block $m$;
			\STATE $\tilde{\mathbf{u}}_{m} \gets \frac{r}{r+1}\tilde{\mathbf{u}}_{m} + \frac{1}{r+1}\bar{\mathbf{u}}_t$;
			\IF{Next round's block is a new block}
				\STATE $m_1 \leftarrow $ Next round's block index;
				\STATE Broadcasts $\bar{\mathbf{w}}_{m_1}$ to all clients: $\mathbf{y}_t^i \gets\bar{\mathbf{w}}_{m_1}$;
			\ENDIF
 		        \ENDIF
			\STATE* {\tt /* Each client $i$'s process */}
 			 \STATE Observes local stochastic gradients $\mathbf{g}_t^i$ and $\mathbf{G}_t^i$ for local models ${\mathbf{x}}_t^i$ and ${\mathbf{y}}_t^i$, respectively:\\
  			$\mathbf{x}_{t+1}^i \gets \mathbf{x}_{t}^i-\gamma \mathbf{g}_t^i$;
   			$\mathbf{y}_{t+1}^i  \gets \mathbf{y}_{t}^i-\eta \mathbf{G}_t^i$;
   	\ENDFOR
   \STATE \bfseries return ${\bf U}$
\end{algorithmic}
\end{algorithm}

The block-specific predictors returned by MM-PSGD only have a convergence guarantee with respective to the single optimal global model. In this subsection, we propose MC-PSGD to further improve the performance of the predictors, requiring that each predictor also converges to the optimal block-specific model. With such a result, MC-PSGD would have better performance when the datasets across the blocks are extremely heterogeneous. We note that a separate model trained only by the block's data would converge to the block's optimal model from the results in learning theory~\cite{DBLP:conf/aaai/Yu19}. With this observation, we augment MM-PSGD with a block-separate training strategy. The basic idea behind MC-PSGD is to evaluate the models for each block from the block-mixed training chain like in MM-PSGD and a new block-separate training chain, and use the ``better'' model (the model with a smaller average local loss) to update the block-specific predictor at each communication iteration.

We sketch MC-PSGD in Algorithm \ref{alg:alg2}. We first initialize the learning rates $\gamma$ and $\eta$ for the block-mixed chain and the block-separate chains, respectively. We maintain a vector ${\bf U}=(\tilde{\mathbf{u}}_1, \tilde{\mathbf{u}}_2, \cdots, \tilde{\mathbf{u}}_M)$ to record the latest block-specific predictors, and ${\bf W}=(\bar{\mathbf{w}}_1, \bar{\mathbf{w}}_2,  \cdots, \bar{\mathbf{w}}_M)$ to record the latest block-separate models. In each communication iteration from the block $m$, the cloud server updates the global block-mixed model $\bar{\mathbf{x}}_t$ and the global block-separate model $\bar{\mathbf{y}}_t$ by aggregating the corresponding local models collected from all clients (Lines 6 and 7), and then pushes new global models $\bar{\mathbf{x}}_t$ and $\bar{\mathbf{y}}_t$ to all clients (Line 8). Each client evaluates the local losses of $\bar{{\bf x}}_t$ and $\bar{{\bf y}}_{t}$ over its local data, and sends them back to the cloud server (Line 9). The cloud server calculates the average local losses of $\bar{{\bf x}}_t$ and $\bar{{\bf y}}_{t}$ over all the clients, and selects the model with a smaller average loss as the interim model $\bar{{\bf u}}_t$ (Line 10). With this information,  the cloud server can update the latest predictor $\tilde{{\bf u}}_{m}$ for the current block $m$ (Line 12).
Before entering a new data block (say, data block $m_1$), we reset the local block-separate model $\mathbf{y}_t^i$  of each client $i$ to the latest global block-separate model $\bar{\mathbf{w}}_{m_1}$ in this block (Lines 13 to 16). After receiving the global block-mixed and block-separate models, each client runs $I$ local SGD steps in parallel until the next communication iteration (Line 18).

Finally, MC-PSGD also returns $M$ predictors. According to Line 12 in Algorithm \ref{alg:alg2}, the final predictor of the block $m$ should be:
\begin{align}
\tilde{\mathbf{u}}_m=\frac{1}{\parallel \mathcal{R}^m \parallel}\sum_{t\in \mathcal{R}^m}\bar{\mathbf{u}}_t.\label{eq:finalmodel2}
\end{align}
Compared with MM-PSGD, MC-PSGD needs to exchange extra model parameters and losses between the clients and the cloud server. Specifically, according to Theorems \ref{item:th1} and \ref{item:th2}, the communication overhead of MC-PSGD is $2 M^{\frac{1}{4}}$ times more than that in MM-PSGD. Given that the number of blocks $M$ is usually small in FL, such communication overhead is acceptable in practical system deployment.

\section{Convergence Analysis}\label{section:proof}

\subsection{Convergence of Multi-Model Parallel SGD}
In this subsection, we bound the gap between the loss of the single optimal global model and the average loss of our final predictors over $M$ blocks. This is achieved by bounding the average loss $F_{m}\left(\bar{\mathbf{x}}_{t_r}\right)$ over all the communication iterations $t_r \in \mathcal{R}$, where $\bar{\mathbf{x}}_{t_r}$ denotes the global model in the communication iteration $t_r$, and $\mathcal{R}$ is the set of all communication iterations, i.e., $\mathcal{R} = \bigcup_{m=1}^M \mathcal {R}^m$. We note that the block index $m$ depends on $t_r$ given in equation (\ref{eq:sample}).

The update of the global model $\bar{{\bf x}}_{t_r}$ is an aggregation of model updates in a series of successive iterations from $t_r-I$ to $t_r-1$, each of which is further an aggregation over the local model updates from all clients. By some calculations, we can have the following relation between the global models from two successive communication iterations:
\begin{align}
\bar{\mathbf{x}}_{t_r} = \bar{\mathbf{x}}_{t_r-I}-\gamma \sum_{t=t_r-I}^{t_r-1} \mathbf{g}_t,
\end{align}
where $\mathbf{g}_t$ is the average of local gradients of all client at iteration $t$, i.e., $\mathbf{g}_t =\sum_{i=1}^N \mathbf{g}^i_t/N$. Updating the model $\bar{\mathbf{x}}_{t_r}$ directly from the model $\bar{\mathbf{x}}_{t_r-I}$ in the previous communication iteration needs to consider the accumulated gradient $\sum_{t=t_r-I}^{t_r-1} \mathbf{g}_t$ from multiple training iterations, rather than a single gradient $\mathbf{g}_{t_r-I}$. Thus, it is challenging to use the traditional convexity analysis technique to bound the average loss $F_m\left(\bar{\mathbf{x}}_{t_r}\right)$ over the communication iterations. In contrast, we observe that the relation between the global models from two successive training iterations is easy to describe and analyze:
\begin{align}
\bar{{\bf x}}_{t+1} = \bar{{\bf x}}_{t}  - \gamma {\bf g}_{t}.
\end{align}
Now, it is feasible to use the single iteration gradient ${\bf g}_t$ and the property of strong convexity to bound the average loss $F_m\left(\bar{\mathbf{x}}_t\right)$ from all the training iterations. We next control the gap between $F_m\left(\bar{\mathbf{x}}_t\right)$ over all training iterations and $F_m\left(\bar{\mathbf{x}}_{t_r}\right)$ over all communication iterations by selecting appropriate hyperparameters. With these two steps, we can obtain the desired bound for $F_m\left(\bar{{\bf x}}_{t_r}\right)$ and the convergence guarantee for the predictors returned by MM-PSGD.

\begin{theorem}\label{item:th1}
By setting $\gamma=\frac{\sqrt{N}}{L\sqrt{T}}$ and $I\le T^{\frac{1}{4}}/N^{\frac{3}{4}}$ in MM-PSGD, when $K>CMN$, we have:
\begin{align}
\frac{1}{M}\sum_{m=1}^{M}\mathbb{E}\left[F_m\left(\tilde{\mathbf{x}}_m\right)\right]-F\left(\mathbf{x}^*\right)\le O\left(\frac{1}{\sqrt{NT}}\right).
\end{align}
\end{theorem}

\begin{proof}[\textbf{Proof Sketch of Theorem \ref{item:th1}}]
We give a proof sketch here and defer the detailed proofs to Appendix \ref{th1-lemma}.

We first introduce a useful notation for the latter analysis:
\begin{align}
\bar{\mathbf{g}}_t = \frac{1}{N}\sum_{i=1}^N \nabla F_m\left(\mathbf{x}_t^i\right).
\end{align}
We then introduce some lemmas.

\begin{lemma}[Regret of one iteration]\label{lemma:th1-1}
Under Assumptions \ref{assum:no1} and \ref{assum:no2}, for each iteration $t$, we have:
\begin{equation}
\begin{aligned}
     &\mathbb{E}\left[F_m\left(\bar{\mathbf{x}}_t\right)\right]-F_m\left(\mathbf{x}^*\right)\\
\le\ &\frac{1}{2\gamma}\left(\mathbb{E}\left[\parallel\bar{\mathbf{x}}_t-\mathbf{x}^*\parallel^2\right]-\mathbb{E}\left[\parallel\bar{\mathbf{x}}_{t+1}-\mathbf{x}^*\parallel^2\right]\right)\\
&+\frac{\gamma}{2}\mathbb{E}\left[\parallel \mathbf{g}_t-\bar{\mathbf{g}}_t\parallel^2\right]
+\frac{L^2}{2\mu N} \sum_{i=1}^{N} \mathbb{E}\left[\parallel\bar{\mathbf{x}}_t-\mathbf{x}_t^i\parallel^2\right]\\
&+\frac{\gamma L^2}{N}\sum_{i=1}^{N}\mathbb{E}\left[\parallel \bar{\mathbf{x}}_t- \mathbf{x}_t^i\parallel^2\right]+\gamma\mathbb{E}\left[\parallel\nabla F_m(\bar{\mathbf{x}}_t)\parallel^2\right].
\end{aligned}
\end{equation}
\end{lemma}

\begin{lemma}[Bounding the variance]\label{lemma:th1-2}
Under Assumption \ref{assum:no3}, it follows that:
\begin{align}
\mathbb{E}\left[\parallel \mathbf{g}_t-\bar{\mathbf{g}}_t\parallel ^2\right] \le \frac{\sigma^2}{N}.
\end{align}
\end{lemma}

\begin{lemma}[Bounding the deviation of local model]\label{lemma:th1-3}
 Under Assumption \ref{assum:no4}, the deviation between the local model ${\bf x}_t^i$ and global model $\bar{{\bf x}}_t$ at each iteration $t$ is bounded by
\begin{align}
\mathbb{E}\left[ \parallel \bar{\mathbf{x}}_t - \mathbf{x}_t^i\parallel^2 \right] \le 4\gamma^2I^2G^2 \notag.
\end{align}
\end{lemma}

\begin{lemma}[Bounding the average of gradients]\label{lemma:th1-5}
Under Assumptions \ref{assum:no2},  \ref{assum:no3},  \ref{assum:no4}, and  \ref{assum:no5},  and choosing $\gamma=\frac{\sqrt{N}}{L\sqrt{T}}$ and $I\le T^{\frac{1}{4}}/N^{\frac{3}{4}}$, we can bound the average of the gradients:
\begin{align}
\frac{1}{T}\sum_{t=1}^T \mathbb{E}\left[ \parallel\nabla F_{m}(\bar{\mathbf{x}}_t)\parallel^2 \right]\le O\left(\frac{\sqrt{T}}{\sqrt{N}K}\right).
\end{align}
\end{lemma}

\begin{lemma}[Bounding the average loss of iterations]\label{lemma:th1-4}
Based on Lemmas \ref{lemma:th1-1}, \ref{lemma:th1-2}, \ref{lemma:th1-3}, and \ref{lemma:th1-5}, and Assumption \ref{assum:no5}, we can bound the gap between the average loss of $\left\{{\bar{{\mathbf{x}}}_t}|t \in [T]\right\}$ from all  iterations and the loss of the single optimal global model:
\begin{equation}
\begin{aligned}
 &\frac{1}{T} \sum_{t=1}^{T} \mathbb{E} \left[F_{m}\left(\bar{\mathbf{x}}_t\right)\right]-F\left(\mathbf{x}^*\right) \\
=\ &\frac{1}{T} \sum_{t=1}^{T} \mathbb{E} \left[F_{m}(\bar{\mathbf{x}}_t)-F_{m}\left(\mathbf{x}^*\right)\right] \\
\le\ &\frac{2B^2}{\gamma T} + \frac{\gamma \sigma^2}{2N} + \frac{2\gamma^2I^2G^2L^2}{\mu} + 4\gamma^3I^2G^2L^2\\
&+O\left(\frac{\gamma \sqrt{T}}{\sqrt{N}K}\right).
\end{aligned}
\end{equation}
\end{lemma}

\begin{lemma}[Bounding the average loss of communication iterations]\label{lemma:th1-6}
Under Assumptions \ref{assum:no2}, \ref{assum:no3}, and \ref{assum:no4}, we can bound the gap between the average loss of the global models $\left\{{\bar{{\mathbf{x}}}_{t_r}}|t_r \in \mathcal{R}\right\}$ from the communication iterations and that of the global models $\left\{{\bar{{\mathbf{x}}}_t}|t \in [T]\right\}$ from all iterations. 
\begin{equation}
\begin{aligned}
   &\frac{1}{\parallel \mathcal{R}\parallel}\sum_{t_r \in \mathcal{R}}\mathbb{E}\left[F_m\left(\bar{\mathbf{x}}_{t_r}\right)\right]\\
\le\ &\frac{1}{T}\sum_{t=1}^T \mathbb{E}\left[F_m\left(\bar{\mathbf{x}}_t\right)\right]+2\gamma^3I^3G^2L^2+\frac{\gamma^2\sigma^2IL }{2N}.
\end{aligned}
\end{equation}
\end{lemma}

By the above lemmas and the convexity of $F_m(\mathbf{x})$, and choosing  $\gamma=\frac{\sqrt{N}}{L\sqrt{T}}$ and $I\le T^{\frac{1}{4}}/N^{\frac{3}{4}}$, when $K>CMN$, we can obtain the convergence guarantee for MM-PSGD:
\begin{equation}
\begin{aligned}
    &\frac{1}{M}\sum_{m=1}^{M}\mathbb{E}\left[F_m\left(\tilde{\mathbf{x}}_m\right)\right]-F\left(\mathbf{x}^*\right) \\
\le\ &\frac{1}{M}\sum_{m=1}^{M}\frac{1}{\parallel \mathcal{R}^m \parallel}\sum_{t_r \in \mathcal{R}^m}\mathbb{E}\left[F_m(\bar{\mathbf{x}}_{t_r})\right]-F\left(\mathbf{x}^*\right) \\
=\ &\frac{1}{\parallel \mathcal{R}\parallel}\sum_{t_r \in \mathcal{R}}\mathbb{E}\left[F_m\left(\bar{\mathbf{x}}_{t_r}\right)\right] - F\left(\mathbf{x}^*\right)\\
\le\ &O\left(\frac{1}{\sqrt{NT}}\right).
\end{aligned}
\end{equation}
\end{proof}

Theorem \ref{item:th1} shows that MM-PSGD converges at a rate of $O(1/\sqrt{NT})$ over block-cyclic data. The convergence rate is independent of the number of blocks $M$, guaranteeing that the performance would not deteriorate as $M$ increases.

\subsection{Convergence of Multi-Chain Parallel SGD}
We now prove the convergence of MC-PSGD. The main operation in MC-PSGD is that in each communication iteration, for each block, we evaluate the models from two chains and select the one with a smaller average local loss to update the predictor. With this operation, the final predictor would outperform the model from either of the two chains (Lemma~\ref{lemma:th2-2}). We further show that the model from the block-mixed chain can achieve the convergence rate of $O(1/\sqrt{NT})$ with respective to the single optimal global model, and the model from the block-separate chain can have the convergence guarantee of  $O(\sqrt{M}/\sqrt{NT})$ with respective to the block's optimal model (Lemma~\ref{lemma:th2-1}). By these steps, we can achieve the convergence result of MC-PSGD.

\begin{figure*}
\centering
\begin{minipage}{0.68\columnwidth}
\centering
\includegraphics[width=0.98\columnwidth]{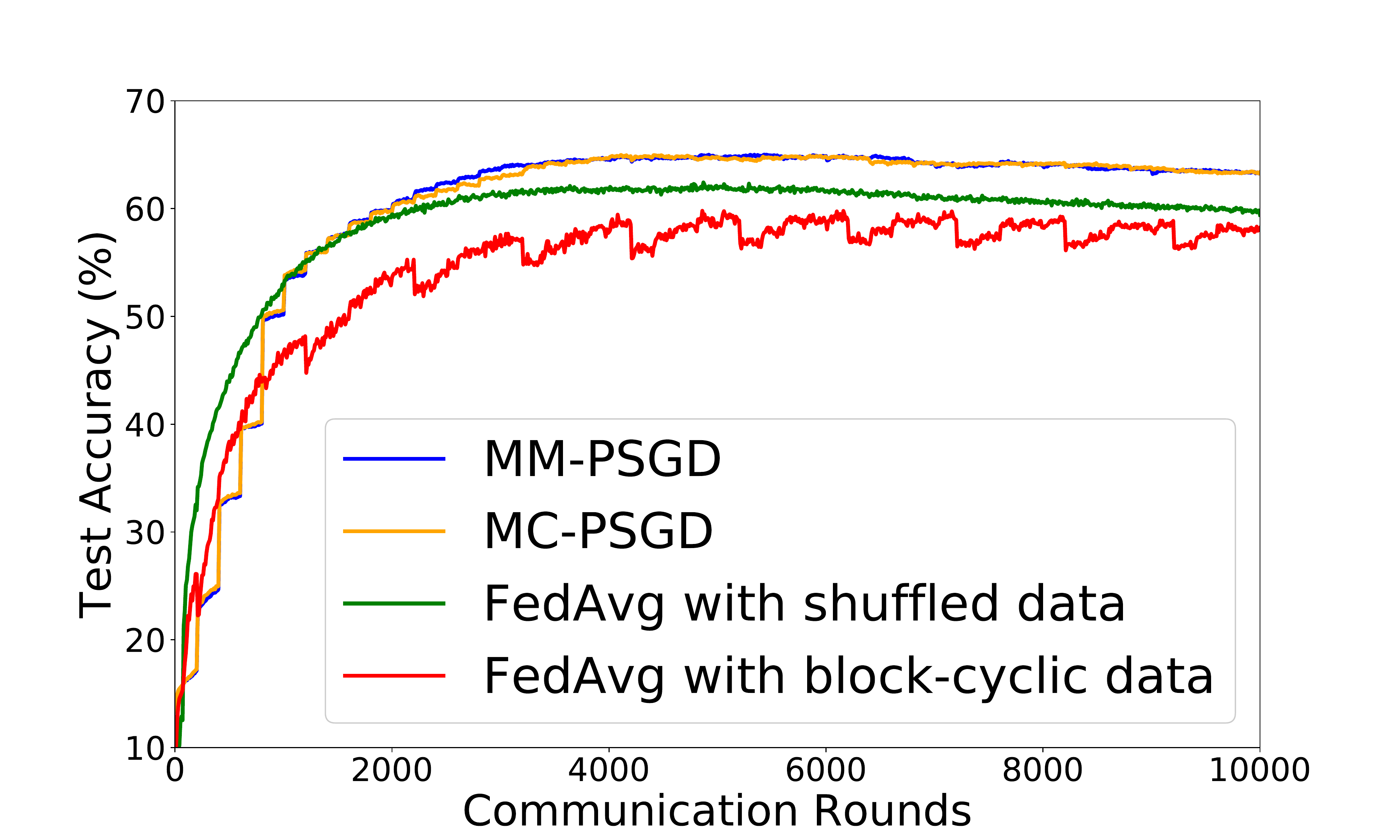}
\caption{Test accuracies of our MM-PSGD and MC-PSGD with block-cyclic data. FedAvg with block-cyclic and shuffled data are introduced as two baselines.} \label{mainfig}
\end{minipage}\hfill
\begin{minipage}{1.36\columnwidth}
\centering
\vskip -0.3in
\subfigure[MM-PSGD]{\label{fig-I-our}
\includegraphics[width=0.48\columnwidth]{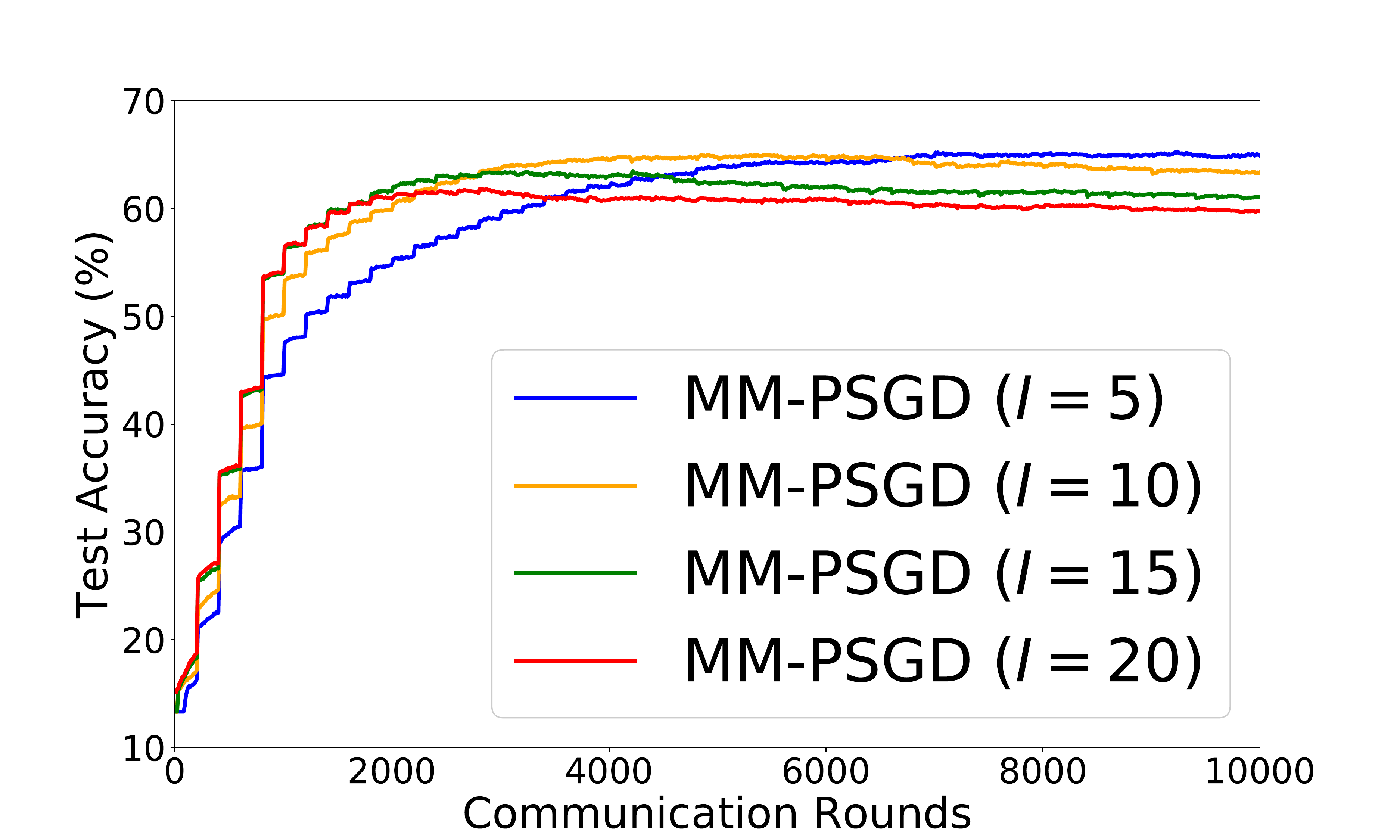}
}
\subfigure[MC-PSGD]{\label{fig-I-fed}
\includegraphics[width=0.48\columnwidth]{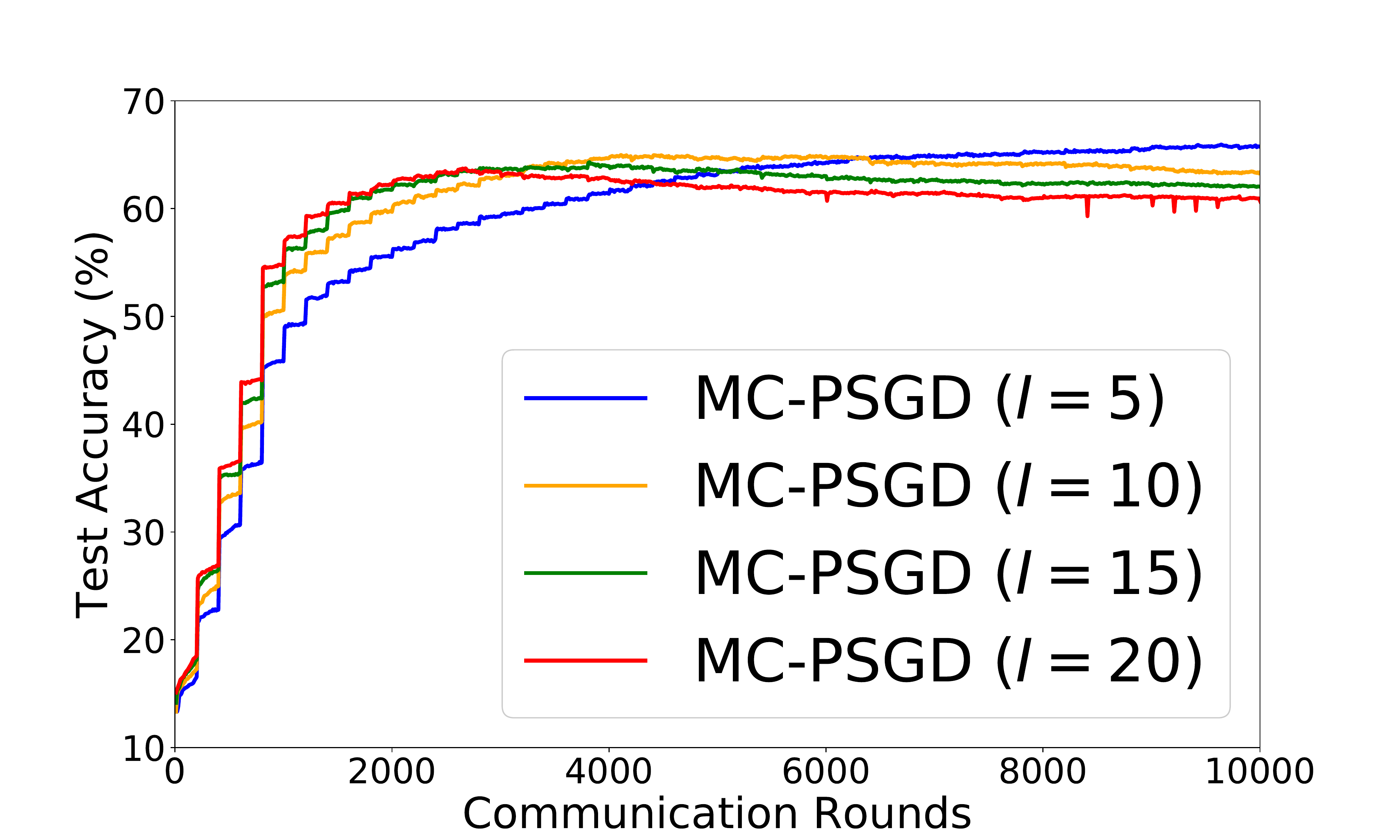}
}%
\caption{Impact of the number of local iterations $I$.}\label{fig-I}
\end{minipage}
\end{figure*}

\begin{figure*}[!t]
\begin{center}
\subfigure[FedAvg]{\label{fig-M-fed}
\includegraphics[width=0.662\columnwidth]{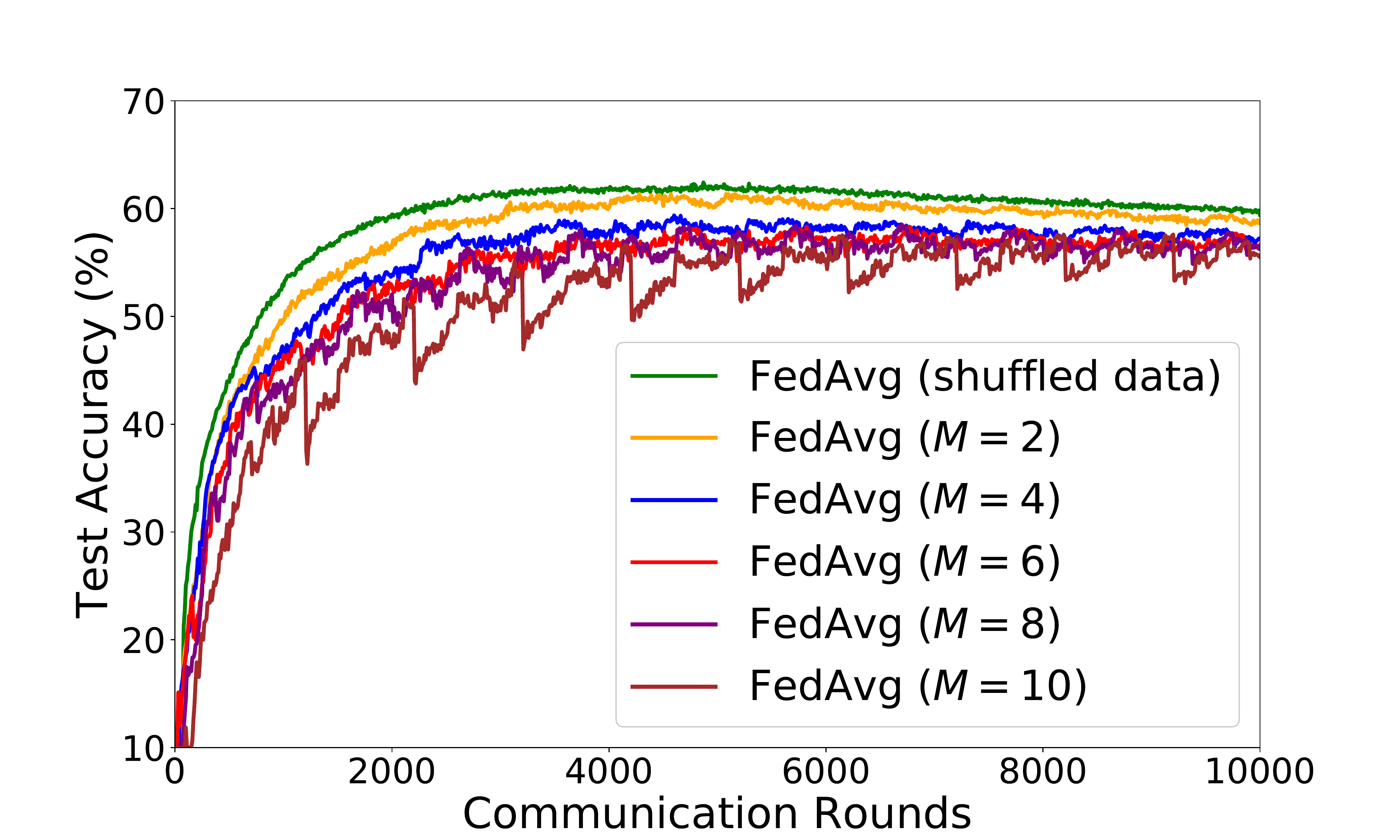}
}
\subfigure[MM-PSGD]{\label{fig-M-multi}
\includegraphics[width=0.662\columnwidth]{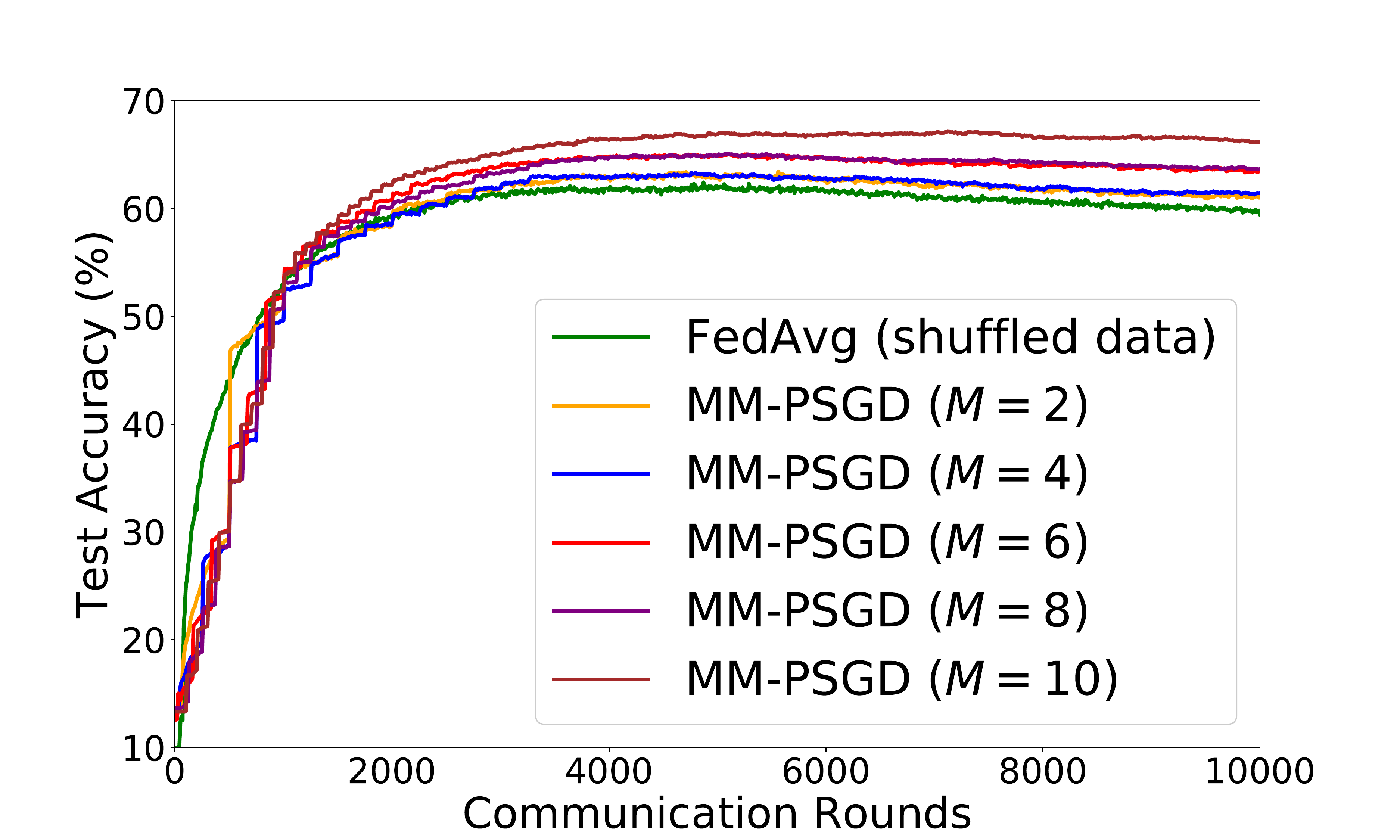}
}
\subfigure[MC-PSGD]{\label{fig-M-2chain}
\includegraphics[width=0.662\columnwidth]{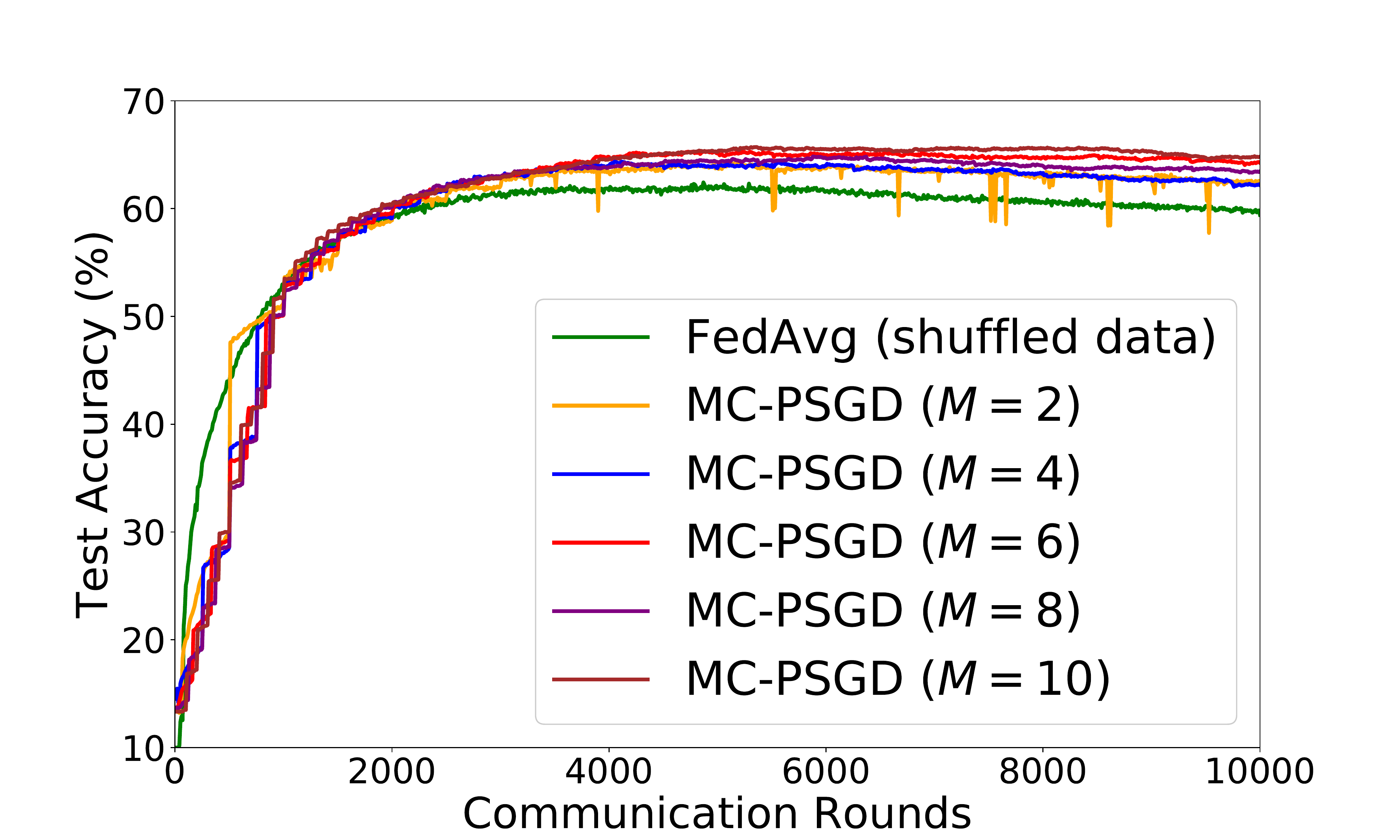}
}
\caption{Impact of the number of blocks $M$. The ideal FedAvg with shuffle data is introduced as a baseline.}\label{fig-M}
\end{center}
\end{figure*}

\begin{theorem}\label{item:th2}
By setting $\gamma=\frac{\sqrt{N}}{L\sqrt{T}}, \eta=\frac{\sqrt{NM}}{L\sqrt{T}}$ and $I\le \frac{(T/M)^{\frac{1}{4}}}{N^{\frac{3}{4}}}$, when $K>CMN$, MC-PSGD has the following convergence results:
\begin{align}
\frac{1}{M}\sum_{m=1}^{M}\mathbb{E}\left[F_m\left(\tilde{\mathbf{u}}_m\right)\right]-F\left(\mathbf{x}^*\right)\le O\left(\frac{1}{\sqrt{NT}}\right),
\end{align}
and for each block $m$,
\begin{align}
\mathbb{E}\left[F_m\left(\tilde{\mathbf{u}}_m\right)-F_{m}\left(\mathbf{x}_m^*\right)\right]
\le
O\left(\frac{\sqrt{M}}{\sqrt{NT}}\right),
\end{align}
where ${\bf x}^*_m$ is the block's optimal model from the block $m$.
\end{theorem}

\begin{proof}
Please refer to Appendix \ref{th2:proof}.
\end{proof}

\section{Experiments}

In this section, we consider an image classification task and present the evaluation results of MM-PSGD and MC-PSGD. We note that although our two algorithms focus on the convex objective in the theoretical part, they still work very well for non-convex problems in practice.

{\bf Model and Dataset.} We take a convolutional neural network (CNN) from PyTorch's tutorial and use the public CIFAR-10 dataset. The CNN is formed by two convolutional layers and three fully connected layers with ReLU activation, max pooling, and softmax. In addition, the CIFAR-10 dataset consists of 10 classes of $32 \times 32$ images with three RGB channels. There are 50,000 and 10,000 images for training and testing, respectively. To simulate the block-cylic data in FL, we first partition both the training and test sets into $M$ heterogeneous blocks based on labels, where a block contains images of several labels, and different blocks may contain partially crossed labels. Each data block is further distributed among $N = 100$ clients in a non-i.i.d. and unbalanced way, where the local training sets on the clients are fetched from the block in sequence, and the sizes of the local training sets roughly follow the normal distribution with a mean of $50000/(MN)$ and a variance of $(10000/(MN))^2$. We implement our algorithms, test each block-specific predictor on the corresponding block's test set, and calculate an average of the test accuracies. 

{\bf Implementation Settings.} For all the experiments, we set the local training batch size to 2. For MM-PSGD, we use a learning rate of $\gamma=0.01$, and for MC-PSGD, we set both learning rates $\gamma = \eta = 0.01$ for the block-mixed and block-separate chains. As the default settings, we set the number of cycles $C$ to 10 and set the number of blocks $M$ to 5. Within each block, we set the total number of rounds $E$ to 200 and let each client run $I = 10$ local iterations of SGD. Furthermore, in our experiments, we observed that later models tend to have better accuracy. Thus, we empirically took the exponentially weighted average of the historical global model to obtain the final predictors, where the base is $0.5$, and the round number works as the exponent.

\textbf{Our Algorithms vs. FedAvg.} We compare MM-PSGD, MC-PSGD, with FedAvg over the block-cyclic data under the default settings. We also evaluate FedAvg over the shuffled data as an ideal baseline, where we directly distribute the randomly shuffled training data among all clients without the operation of block partition.  

We plot the evaluation results in Figure \ref{mainfig} and can observe that: (1) MM-PSGD and MC-PSGD achieve the best test accuracy of 65\%; (2) FedAvg with block-cyclic data does not converge at all, fluctuating between 56\% and 59\%; and (3) FedAvg with the shuffled data achieves the best test accuracy of 62\%, 3\% lower than our algorithms. We can also clearly observe that when the number of rounds is smaller than 3000, the test accuracies of our algorithms grow in a stair shape. This is because the performance of any block-specific predictor can be significantly improved in the first few cycles. 

\begin{table}[!t]
	\caption{Numbers of rounds required to achieve an accuracy of $60\%$ under different numbers of local iterations $I$'s.}\label{tb:expI}
	\begin{center}
		\resizebox{.9\columnwidth}{!}{
			\begin{tabular}[t]{lcccc}
				\toprule
				$I$ & 5 & 10 & 15 & 20\\
				\midrule
				MM-PSGD & 3,210 & 2,010 & 1,610 & 1,610\\
				MC-PSGD & 3,210 & 2,010 & 1,610 & 1,410\\		
				\bottomrule
			\end{tabular}
		}		
	\end{center}
	\vskip -0.1in
\end{table}

{\bf Number of Local Iterations.} We expect that our algorithms would converge in fewer rounds if we choose a larger number of local iterations $I$. This is because the convergence rate is $O(1/\sqrt{NT})$, where $T$ denotes the total number of local iterations, and the total number of rounds should be $T/I$. However, increasing $I$ will enlarge the bias between the local models and the global model, so the number of rounds needed for convergence may not consistently decrease inversely with $I$.
 
We evaluate MM-PSGD and MC-PSGD under different numbers of local iterations $I$'s, where $I$ increases from 5 to 20 with a step of 5. Figure \ref{fig-I} plots the test accuracies, and Table \ref{tb:expI} lists the number of rounds required to first achieve the test accuracy of 60\%. We can see that both MM-PSGD and MC-PSGD converge faster with a larger $I$, and MC-PSGD converges faster than MM-PSGD at $I=20$. These results conform to our expectation and analysis.
 
{\bf Number of Blocks.} We expect that the performance of our algorithms would not deteriorate as the number of blocks $M$ increase. We let $M$ take different values for comparison, while fixing the total number of rounds in each cycle at 1000 (i.e., $M \times E = 1000$). Figure \ref{fig-M} plots the results of our algorithms as well as two baselines of FedAvg with cyclic data and FedAvg with shuffled data.

From Figure~\ref{fig-M-fed}, we can see that the performance of FedAvg significantly deteriorates with a larger $M$. In particular, the best accuracy of FedAvg is 61\% at $M=2$. But at $M=10$, its accuracy fluctuates between 53\% and 57\%. From Figures \ref{fig-M-multi} and \ref{fig-M-2chain}, we can see that as $M$ increases, both of our algorithms perform even better. For MM-PSGD, it achieves the best test accuracies of 63\% and 67\% at $M=2$ and $M=10$, respectively. For MC-PSGD, its performance is more stable, achieving the best test accuracy between $64\%$ and $65\%$ regardless of $M$. We also observe some points with little fluctuation in Figure \ref{fig-M-2chain}, which is reasonable in non-convex optimization. Two chains for each block may converge to different local optimal models, averaging them may cause that the predictor is neither of the local optimal models.

\begin{figure}[!t]
\begin{center}
\subfigure[MM-PSGD]{\label{fig-frac-multi}
\includegraphics[width=0.75\columnwidth]{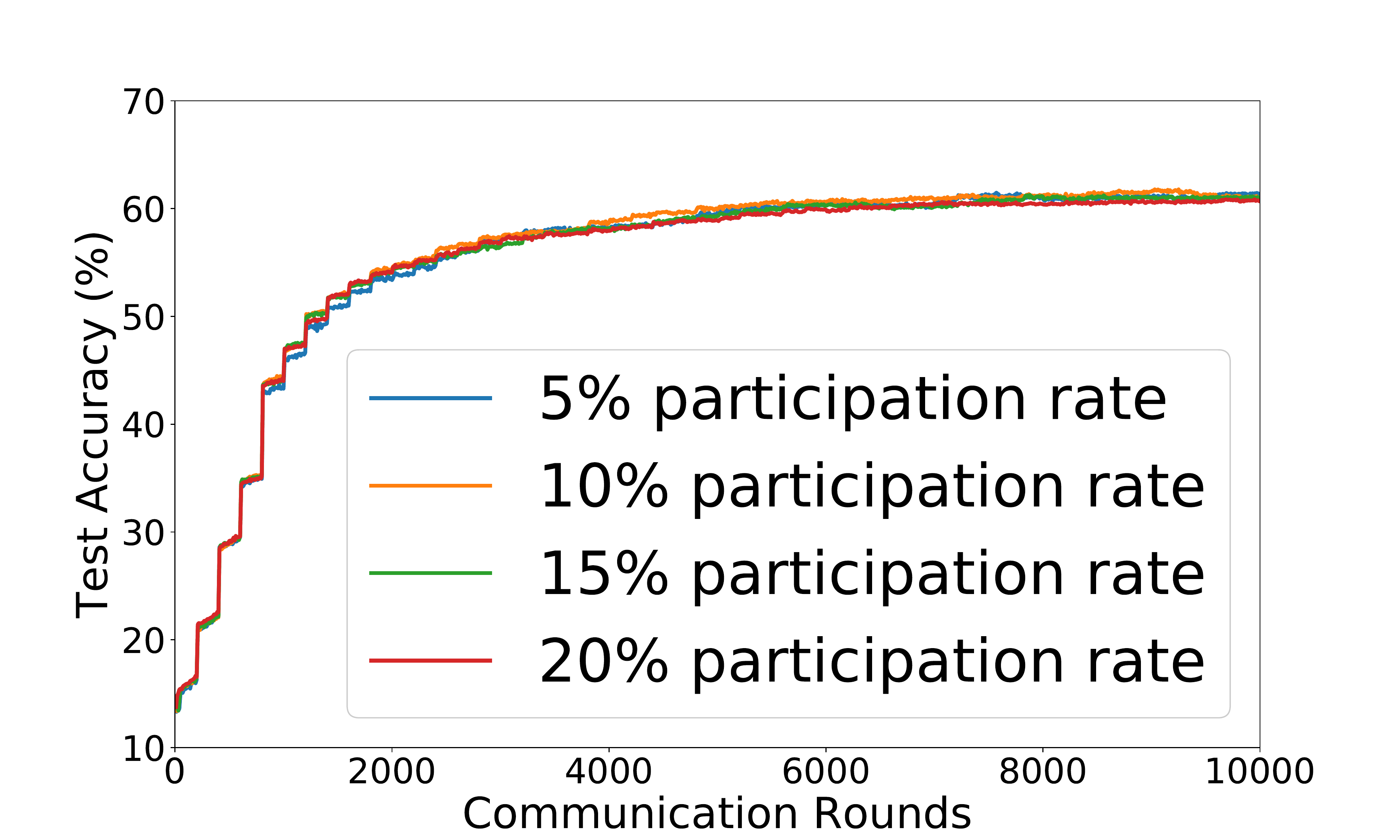}
}
\subfigure[MC-PSGD]{\label{fig-frac-2chain}
\includegraphics[width=0.75\columnwidth]{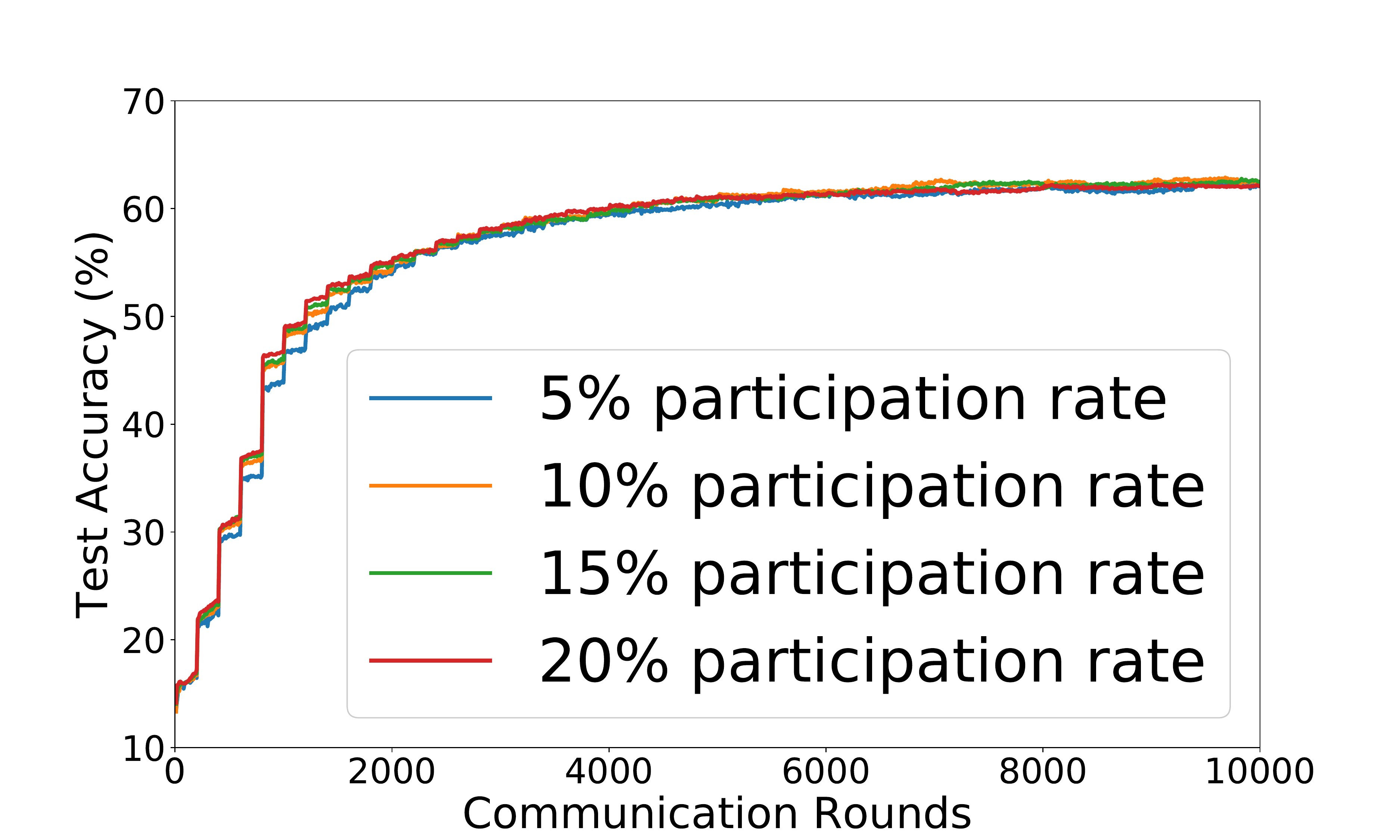}
}
\caption{Impact of the participation rate of clients.} \label{fig-frac}
\end{center}
\vskip -0.2in
\end{figure}

{\bf Participation Rate of Clients.} In the practical scenario of cross-device FL, only part of clients are chosen to participate in each round of collaborative training. In this set of simulations, we set the number of total clients to 1000 and randomly select a certain fraction of clients to participate in each round. Now, the clients in our algorithms denote those participating ones rather than the whole client pool. Figure~\ref{fig-frac} shows the evaluation results, when the participation rate increases from 5\% to 20\% with a step of 5\%. We can see that our two algorithms are both robust to the participation rate. In particular, the best test accuracies of MM-PSGD are between 60\% and 61\%, whereas the best test accuracies of MC-PSGD are higher, between 61\% and 62\%.

\section{Conclusion}
In this work, we considered unbalanced, non-i.i.d., and block-cyclic data distributions in FL. 
Such data characteristics would deteriorate the performance of conventional FL algorithms. To handle the problems introduced by cyclic data, we proposed MM-PSGD and MC-PSGD to obtain a series of block-specific predictors. Both MM-PSGD and MC-PSGD attain a convergence guarantee of $O(1/\sqrt{NT})$, achieving a linear speedup with respect to the number of clients. MC-PSGD can further guarantee that each block-specific predictor converges to the block's optimal model at a rate of $O(\sqrt{M}/\sqrt{NT})$, while adding an acceptable communication overhead. Empirical studies over the CIFAR-10 dataset demonstrate effectiveness and robustness.

\bibliography{cycle}
\bibliographystyle{icml2020}

\onecolumn
\appendix

\section{Proof of Theorem \ref{item:th1}}\label{th1-lemma}
\begin{proof}[\textbf{Proof of Lemma \ref{lemma:th1-1}}]
For any static model $\mathbf{x}$, we have
\begin{equation}
\begin{aligned}
\parallel \bar{\mathbf{x}}_{t+1}-\mathbf{x} \parallel^2 &=\parallel\bar{\mathbf{x}}_t - \gamma \mathbf{g}_t - \mathbf{x}\parallel^2\\
&=\parallel\bar{\mathbf{x}}_t-\gamma \bar{\mathbf{g}}_t - \mathbf{x}+\gamma \bar{\mathbf{g}}_t-\gamma \mathbf{g}_t\parallel^2 \\ 
&=\underbrace{\parallel\bar{\mathbf{x}}_t-\gamma \bar{\mathbf{g}}_t - \mathbf{x}\parallel^2}_{A_1}
+
\gamma^2\parallel \bar{\mathbf{g}}_t - \mathbf{g}_t \parallel^2
+
\underbrace{2\gamma\left<\bar{\mathbf{x}}_t-\gamma \bar{\mathbf{g}}_t - \mathbf{x},\bar{\mathbf{g}}_t-\mathbf{g}_t \right>}_{A_2}.
\end{aligned}\label{eq:s1-1}
\end{equation}
Note that $\mathbb{E}\left[A_2\right]=0$.
We next focus on bounding $A_1$ and split it in to three terms.
\begin{align}
A_1
=
\parallel\bar{\mathbf{x}}_t-\mathbf{x}\parallel^2
\underbrace{-2\gamma \left<\bar{\mathbf{x}}_t-\mathbf{x},\bar{\mathbf{g}}_t\right>}_{B_1}
+
\underbrace{\gamma^2\parallel\bar{\mathbf{g}}_t\parallel^2}_{B_2}. \label{eq:s1-2}
\end{align}
Note that 
\begin{equation}
\begin{aligned}
B_1 
&= -2\gamma\left<\bar{\mathbf{x}}_t-\mathbf{x},  \frac{1}{N}\sum_{i=1}^{N} \nabla F_m^i\left(\mathbf{x}_t^i\right)\right> \\
&= - \frac{2\gamma}{N}\sum_{i=1}^{N}\left<\bar{\mathbf{x}}_t-\mathbf{x},\nabla F_m^i\left(\mathbf{x}_t^i\right)\right> \\
&=- \frac{2\gamma}{N}\sum_{i=1}^{N}\left<\bar{\mathbf{x}}_t-\mathbf{x},\nabla F_m^i\left(\bar{\mathbf{x}}_t\right)+\nabla F_m^i\left(\mathbf{x}_t^i\right)-\nabla F_m^i\left(\bar{\mathbf{x}}_t\right)\right>\\
&= \frac{2\gamma}{N}\sum_{i=1}^{N}\left[
		\underbrace{-\left<\bar{\mathbf{x}}_t-\mathbf{x},\nabla F_m^i\left(\bar{\mathbf{x}}_t\right)\right>}_{C_1}
		\underbrace{-\left<\bar{\mathbf{x}}_t-\mathbf{x},\nabla F_m^i\left(\mathbf{x}_t^i\right)-\nabla F_m^i\left(\bar{\mathbf{x}}_t\right)\right>}_{C_2}
		\right].
\end{aligned}\label{eq:s1-3}
\end{equation}
where $m$ depends on $t$ given in equation (\ref{eq:sample}).

By the $\mu$-convexity of $F_m^i(\mathbf{x})$, we have
\begin{align}
C_1 \le
-\left(F_m^i\left(\bar{\mathbf{x}}_t\right)-F_m^i\left(\mathbf{x}\right)\right)-\frac{\mu}{2}\parallel\bar{\mathbf{x}}_t-\mathbf{x}\parallel^2.\label{eq:s1-4}
\end{align}
By the basic inequality $-2\left<{\bf a},{\bf b}\right>\le \alpha\parallel {\bf a}\parallel^2+\frac{1}{\alpha}\parallel {\bf b} \parallel^2\ (\alpha>0)$, we have
\begin{align}
C_2
\le
\frac{\mu}{2}\parallel\bar{\mathbf{x}}_t-\mathbf{x}\parallel^2+\frac{1}{2\mu}\parallel\nabla F_m^i\left(\mathbf{x}_t^i\right)-\nabla F_m^i\left(\bar{\mathbf{x}}_t\right)\parallel^2. \label{eq:s1-5}
\end{align}
Substituting equations (\ref{eq:s1-4}) and (\ref{eq:s1-5}) into (\ref{eq:s1-3}) and using the $L$-smoothness of $F_m^i(\mathbf{x})$, we have
\begin{equation}
\begin{aligned}
B_1
&\le \frac{2\gamma}{N}\sum_{i=1}^N \left[-\left(F_m^i\left(\bar{\mathbf{x}}_t\right)-F_m^i\left(\mathbf{x}\right)\right)
+\frac{1}{2\mu}\parallel\nabla F_m^i\left(\mathbf{x}_t^i\right)-\nabla F_m^i\left(\bar{\mathbf{x}}_t\right)\parallel^2 \right]\\
&\le -2\gamma\left(F_m\left(\bar{\mathbf{x}}_t\right)-F_m\left(\mathbf{x}\right)\right)+\frac{\gamma L^2}{\mu N}\sum_{i=1}^{N}\parallel\bar{\mathbf{x}}_t-\mathbf{x}_t^i\parallel^2 .
\end{aligned}\label{eq:s1-6}
\end{equation}
We next focus on bounding $B_2$: 
\begin{equation}
\begin{aligned}
B_2&=
\gamma^2 \parallel \frac{1}{N}\sum_{i=1}^{N}\nabla F_m^i\left(\mathbf{x}_t^i\right)\parallel ^2 \\
&=
\gamma^2 \parallel \frac{1}{N}\sum_{i=1}^{N}\left[\nabla F_m^i\left(\mathbf{x}_t^i\right)-\nabla F_m^i\left(\bar{\mathbf{x}}_t\right)+\nabla F_m^i\left(\bar{\mathbf{x}}_t\right)\right]\parallel^2  \\ 
&\overset{(a)}\le 2\gamma^2 \parallel \frac{1}{N}\sum_{i=1}^{N}\left[\nabla F_m^i\left(\mathbf{x}_t^i\right)-\nabla F_m^i\left(\bar{\mathbf{x}}_t\right)\right]\parallel ^2 + 2\gamma^2 \parallel \frac{1}{N}\sum_{i=1}^{N}\nabla F_m^i\left(\bar{\mathbf{x}}_t\right)\parallel^2 \\
&\overset{(b)}{\le} \frac{2\gamma^2 L^2}{N}\sum_{i=1}^{N}\parallel \bar{\mathbf{x}}_t-\mathbf{x}_t^i\parallel^2+2\gamma^2\parallel\nabla F_m(\bar{\mathbf{x}}_t)\parallel^2 ,
\end{aligned}\label{eq:s1-7}
\end{equation}
where (a) follows from $2\left<{\bf a},{\bf b}\right>\le \parallel {\bf a}\parallel^2+\parallel {\bf b} \parallel^2$, and (b) follows from the inequality $\parallel \sum_{i=1}^N {\bf z}_i\parallel^2 \le N\sum_{i=1}^N \parallel {\bf z}_i \parallel^2$ and the $L$-smoothness of $F_m^i(\mathbf{x})$.

By substituting equations (\ref{eq:s1-3}) and (\ref{eq:s1-7}) into (\ref{eq:s1-2}), we have
\begin{equation}
A_1
\le \parallel\bar{\mathbf{x}}_t-\mathbf{x}\parallel^2-2\gamma(F_m(\bar{\mathbf{x}}_t)-F_m(\mathbf{x}))
+\frac{\gamma L^2}{\mu N}\sum_{i=1}^{N}\parallel\bar{\mathbf{x}}_t-\mathbf{x}_t^i\parallel^2 
+\frac{2\gamma^2 L^2}{N}\sum_{i=1}^{N}\parallel \bar{\mathbf{x}}_t - \mathbf{x}_t^i \parallel^2
+2\gamma^2\parallel\nabla F_m\left(\bar{\mathbf{x}}_t\right)\parallel^2.\label{eq:s1-8}
\end{equation}
We substitute equation (\ref{eq:s1-8}) into (\ref{eq:s1-1}), let $\mathbf{x}=\mathbf{x}^*$, rerrange the terms, and finally take expectations on both sides
\begin{equation}
\begin{aligned}
\mathbb{E}\left[F_m(\bar{\mathbf{x}}_t)\right]-F_m(\mathbf{x}^*)
\le&\ \frac{1}{2\gamma}\left(\mathbb{E}\left[\parallel\bar{\mathbf{x}}_t-\mathbf{x}^*\parallel^2\right]-\mathbb{E}\left[\parallel\bar{\mathbf{x}}_{t+1}-\mathbf{x}^*\parallel^2\right]\right) 
+\frac{\gamma}{2}\mathbb{E}\left[\parallel \mathbf{g}_t-\bar{\mathbf{g}}_t\parallel^2\right]\\
&\ +\frac{L^2}{2\mu N}\sum_{i=1}^{N}\mathbb{E}\left[\parallel\bar{\mathbf{x}}_t-\mathbf{x}_t^i\parallel^2\right]
+\frac{\gamma L^2}{N}\sum_{i=1}^{N}\mathbb{E}\left[\parallel \bar{\mathbf{x}}_t- \mathbf{x}_t^i\parallel^2\right]+\gamma\mathbb{E}\left[\parallel\nabla F_m(\bar{\mathbf{x}}_t)\parallel^2\right].
\end{aligned}\label{eq:s1-9}
\end{equation}
\end{proof}

\begin{proof}[\textbf{Proof of Lemma \ref{lemma:th1-2}}]
From Assumption \ref{assum:no3}, the variance of stochastic gradient is bounded by $\sigma^2$, then
\begin{equation}
\begin{aligned}
\mathbb{E} \left[\parallel \mathbf{g}_t - \bar{\mathbf{g}}_t \parallel^2\right]
&= \mathbb{E} \left[\parallel\frac{1}{N}\sum_{i=1}^{N}\left(\nabla f\left(\mathbf{x}_t^i,\xi_t^i \right)-\nabla F_m^i\left(\mathbf{x}_t^i\right)\right)\parallel^2\right] \\
&\overset{(a)}=
\frac{1}{N^2}\sum_{i=1}^{N}\mathbb{E}\left[\parallel \nabla f\left(\mathbf{x}_t^i,\xi_t^i\right)-\nabla F_m^i\left(\mathbf{x}_t^i\right)\parallel^2 \right]
\le \frac{\sigma^2}{N},
\end{aligned}\label{eq:s2-1}
\end{equation}
where (a) follows from that $\nabla f\left(\mathbf{x}_t^i,\xi_t^i \right)-\nabla F_m^i\left(\mathbf{x}_t^i\right)$ has mean ${0}$ and is independent across clients.
\end{proof}

\begin{proof}[\textbf{Proof of Lemma \ref{lemma:th1-3}}]
For any $0<t\le T$, there exists a largest $t_0\le t$ such that $\bar{\mathbf{x}}_{t_0}=\mathbf{x}_{t_0}^i$ for $i=1,2,\cdots,N$. Then, we have
\begin{equation}
\begin{aligned}
\mathbb{E}\left[\parallel \bar{\mathbf{x}}_t - \mathbf{x}_t^i\parallel^2\right]
&= \mathbb{E}\left[\parallel \gamma \sum_{\tau=t_0}^{t-1} \frac{1}{N}\sum_{i=1}^{N}\mathbf{g}_t^i-\gamma  \sum_{\tau=t_0}^{t-1}\mathbf{g}_t^i \parallel ^2\right]\\
&= \gamma^2 \mathbb{E}\left[\parallel\sum_{\tau=t_0}^{t-1} \frac{1}{N}\sum_{i=1}^{N}\mathbf{g}_t^i-\sum_{\tau=t_0}^{t-1}\mathbf{g}_t^i\parallel^2\right]\\
&\overset{(a)}\le 2\gamma^2 \mathbb{E}\left[\parallel\sum_{\tau=t_0}^{t-1} \frac{1}{N}\sum_{i=1}^{N}\mathbf{g}_t^i\parallel^2+\parallel\sum_{\tau=t_0}^{t-1}\mathbf{g}_t^i\parallel^2\right]\\
&\overset{(b)}\le 2\gamma^2\left(t-t_0\right)\mathbb{E}\left[\sum_{\tau=t_0}^{t-1}\parallel\frac{1}{N}\sum_{i=1}^{N}\mathbf{g}_t^i\parallel^2+\sum_{\tau=t_0}^{t-1}\parallel \mathbf{g}_t^i\parallel^2\right]\\
&\overset{(c)} \le 2\gamma^2\left(t-t_0\right)\mathbb{E}\left[\sum_{\tau=t_0}^{t-1}\left(\frac{1}{N}\sum_{i=1}^{N}\parallel\mathbf{g}_t^i\parallel^2\right)+\sum_{\tau=t_0}^{t-1}\parallel \mathbf{g}_t^i\parallel^2\right]\\
&\overset{(d)}\le 4\gamma^2I^2G^2 .
\end{aligned}\label{eq:s3-1}
\end{equation}
where (a)-(c) follow from the inequality $\parallel \sum_{i=1}^N {\bf z}_i\parallel^2 \le N\sum_{i=1}^N \parallel {\bf z}_i \parallel^2$, and (d) follows from Assumption \ref{assum:no4}.
\end{proof}

\begin{proof}[\textbf{Proof of Lemma \ref{lemma:th1-5}}]
We focus on bounding the average of gradients. Although the sampling is block-cyclic, when we focus on a certain block $m$ in a certain cycle $c$, it is equal to the non-cyclic case but with only $K$ iterations.

For block $m$ in cycle $c$ , we let $t_c^m$ denote the start of that block. According to \citet{DBLP:conf/aaai/Yu19}, we have
\begin{equation}
\begin{aligned}
\frac{1}{K}\sum_{t=t_c^m}^{t_c^m+K-1} \mathbb{E}\left[\parallel\nabla F_m\left(\bar{\mathbf{x}}_t\right)\parallel^2 \right]
\le &\frac{2\left(F_m\left(\bar{\mathbf{x}}_{t_c^m}\right)-F_m\left(\bar{\mathbf{x}}_{t_c^m+K}\right)\right)}{\gamma K}+4\gamma ^2I^2G^2L^2+\frac{L}{N}\gamma \sigma^2 \\
\overset{(a)}{\le} &\frac{4\left(B^2L^2 + B^2 + B^2L\right)}{\gamma K}+4\gamma ^2I^2G^2L^2+\frac{L}{N}\gamma \sigma^2,
\end{aligned}\label{eq:s5-1}
\end{equation}
where (a) follows from Lemma \ref{lemma:s1-1}.
So the average $l_2$-norm of gradients should be bounded by:
\begin{equation}
\begin{aligned}
\frac{1}{T}\sum_{t=1}^T \mathbb{E}\left[\parallel\nabla F_{m}(\bar{\mathbf{x}}_t)\parallel^2\right]
&= \frac{1}{CM}\sum_{c,m}\frac{1}{K}\sum_{t=t_c^m}^{t_c^m+K-1} \mathbb{E}\left[\parallel \nabla F_m(\bar{\mathbf{x}}_t)\parallel^2\right]\\
&\le \frac{4\left(B^2L^2 + B^2 + B^2L\right)}{\gamma K}+4\gamma ^2I^2G^2L^2+\frac{L}{N}\gamma \sigma^2.
\end{aligned}\label{eq:s5-1.5}
\end{equation}

If $\gamma=\frac{\sqrt{N}}{L\sqrt{T}}\le 1$ and $I\le T^{\frac{1}{4}}/N^{\frac{3}{4}}$, we have
\begin{align}
\frac{1}{T}\sum_{t=1}^T \mathbb{E}\left[\parallel\nabla F_{m}(\bar{\mathbf{x}}_t)\parallel^2\right]
\le O\left(\frac{\sqrt{T}}{\sqrt{N}K}\right).\label{eq:s5-2}
\end{align}
\end{proof}

\begin{proof}[\textbf{Proof of Lemma \ref{lemma:th1-4}}]
We sum over $t\in \{1,2, \cdots , T\}$ for equation (\ref{eq:s1-9}), divide $T$ on both sides, and substitute equations (\ref{eq:s2-1}) and (\ref{eq:s3-1}). Then, we have
\begin{equation}
\begin{aligned}
\frac{1}{T} \sum_{t=1}^{T} \mathbb{E} \left[F_{m}(\bar{\mathbf{x}}_t)-F_{m}\left(\mathbf{x}^*\right)\right]
\le&
\frac{\mathbb{E}\left[\parallel \bar{\mathbf{x}}_1-\mathbf{x}^* \parallel ^2\right]-\mathbb{E}\left[\parallel \bar{\mathbf{x}}_{T+1}-\mathbf{x}^*\parallel ^2\right]}{2\gamma T} 
+ \frac{\gamma \sigma^2}{2N} 
+ \frac{2\gamma^2I^2G^2L^2}{\mu} \\
&+ 4\gamma^3I^2G^2L^2 
+ \frac{\gamma}{T} \sum_{t=1}^T \mathbb{E}\left[\parallel\nabla F_m(\bar{\mathbf{x}}_t)\parallel^2\right].
\end{aligned}\label{eq:s4-1}
\end{equation}
Under Assumption \ref{assum:no5}, we have
\begin{align}
\frac{\mathbb{E}\left[\parallel \bar{\mathbf{x}}_1-\mathbf{x}^* \parallel ^2\right]-\mathbb{E}\left[\parallel \bar{\mathbf{x}}_{T+1}-\mathbf{x}^*\parallel ^2\right]}{2\gamma T} 
\le
\frac{2B^2}{\gamma T}.\label{eq:s4-2}
\end{align}
According to Lemma \ref{lemma:th1-5}, we have 
\begin{align}
\frac{1}{T} \sum_{t=1}^{T} \mathbb{E} \left[F_{m}(\bar{\mathbf{x}}_t)-F_{m}\left(\mathbf{x}^*\right)\right]
\le&
\frac{2B^2}{\gamma T} 
+ \frac{\gamma \sigma^2}{2N}
+ \frac{2\gamma^2I^2G^2L^2}{\mu}
+ 4\gamma^3I^2G^2L^2
+ O\left(\frac{\gamma \sqrt{T}}{\sqrt{N}K}\right).\label{eq:s4-3}
\end{align}
\end{proof}

\begin{proof}[\textbf{Proof of Lemma \ref{lemma:th1-6}}]
To obtain the required result, we need to bound the gap between the loss of rounds and the loss of iterations.
From the $L$-smoothness of the functions $F_m^i(\mathbf{x})$ and Assumption \ref{assum:no2}, we have
\begin{align}
\mathbb{E}\left[F_m\left(\bar{\mathbf{x}}_t\right)-F_m\left(\bar{\mathbf{x}}_{t-1}\right)\right]
\le
\underbrace{\mathbb{E}\left[\left<\bar{\mathbf{x}}_t-\bar{\mathbf{x}}_{t-1},\nabla F_m\left(\bar{\mathbf{x}}_{t-1}\right)\right>\right]}_{D_1}
+
\frac{L}{2}\underbrace{\mathbb{E}\left[\parallel\bar{\mathbf{x}}_t-\bar{\mathbf{x}}_{t-1}\parallel ^2\right]}_{D_2}.\label{eq:s6-1}
\end{align}
We first focus on $D_2$. Under Assumption \ref{assum:no3}, we have
\begin{equation}
\begin{aligned}
\mathbb{E}\left[\parallel\bar{\mathbf{x}}_{t}-\bar{\mathbf{x}}_{t-1}\parallel^2\right]
&= \gamma^2 \mathbb{E}\left[\parallel\frac{1}{N}\sum_{i=1}^{N}\mathbf{g}_{t-1}^i \parallel^2\right]\\
&\overset{(a)}= \gamma^2 \mathbb{E}\left[\parallel\frac{1}{N}\sum_{i=1}^{N}\left(\mathbf{g}_{t-1}^i-\nabla F_m^i\left(\mathbf{x}_{t-1}^i\right)\right)\parallel^2\right]
+\gamma^2\mathbb{E}\left[\parallel\frac{1}{N}\sum_{i=1}^N\nabla F_m^i\left(\mathbf{x}_{t-1}^i\right)\parallel^2\right]\\
&\overset{(b)}= \frac{\gamma^2}{N^2} \sum_{i=1}^{N}\mathbb{E}\left[\parallel\mathbf{g}_{t-1}^i-\nabla F_m^i\left(\mathbf{x}_{t-1}^i\right)\parallel^2\right]
+\gamma^2\mathbb{E}\left[\parallel\frac{1}{N}\sum_{i=1}^N\nabla F_m^i\left(\mathbf{x}_{t-1}^i\right)\parallel^2\right]\\
&\le \frac{\gamma^2 \sigma^2}{N} +\gamma^2 \mathbb{E}\left[\parallel\frac{1}{N} \sum_{i=1}^N  \nabla F_m^i\left(\mathbf{x}_{t-1}^i\right)\parallel^2\right],
\end{aligned}\label{eq:s6-2}
\end{equation}
where (a) follows from $\mathbb{E}\left[\mathbf{g}_{t-1}^i\right]=\nabla F_m^i\left({\bf x}_{t-1}^i\right)$ and the basic inequality $\mathbb{E}\left[\parallel \mathbf{z} \parallel^2\right]=\mathbb{E}\left[\parallel \mathbf{z}-\mathbb{E}\left[\mathbf{z}\right] \parallel^2 \right]+\parallel\mathbb{E}\left[\mathbf{z}\right]\parallel^2$ for any vector $\mathbf{z}$,
and (b) follows from that $\mathbf{g}_{t-1}^i-\nabla F_m^i\left({\bf x}_{t-1}^i\right)$ has mean $0$ and is independent across clients.

We then focus on $D_1$:
\begin{equation}
\begin{aligned}
&\mathbb{E}\left[\left<\bar{\mathbf{x}}_t-\bar{\mathbf{x}}_{t-1},\nabla F_m\left(\bar{\mathbf{x}}_{t-1}\right)\right>\right] \\
=&-\gamma \mathbb{E}\left<\frac{1}{N}\sum_{i=1}^{N}\nabla F_m^i\left(\mathbf{x}_{t-1}^i\right),\nabla F_m\left(\bar{\mathbf{x}}_{t-1}\right)\right>\\
=&-\frac{\gamma}{2}\mathbb{E}\left[\parallel\frac{1}{N}\sum_{i=1}^{N}\nabla F_m^i\left(\mathbf{x}_{t-1}^i\right)\parallel^2 + \parallel\nabla F_m\left(\bar{\mathbf{x}}_{t-1}\right)\parallel^2 \right]
+\frac{\gamma}{2}
\underbrace{\mathbb{E}\left[\parallel\nabla F_m\left(\bar{\mathbf{x}}_{t-1}\right)-\frac{1}{N}\sum_{i=1}^{N}\nabla F_m^i\left(\mathbf{x}_{t-1}^i\right)\parallel^2 \right]}_{E}.
\end{aligned}\label{eq:s6-3}
\end{equation}

Under Assumptions \ref{assum:no2} and \ref{assum:no4} and Lemma \ref{lemma:th1-3}, we bound $E$:
\begin{equation}
\begin{aligned}
E &= \mathbb{E}\left[\parallel \frac{1}{N}\sum_{i=1}^{N}\nabla F_m\left(\bar{\mathbf{x}}_{t-1}\right)-\frac{1}{N}\sum_{i=1}^{N}\nabla F_m^i\left(\mathbf{x}_{t-1}^i\right)\parallel^2\right]\\
    &= \frac{1}{N^2}\mathbb{E}\left[\parallel\sum_{i=1}^{N}\left(\nabla F_m\left(\bar{\mathbf{x}}_{t-1}\right)-\nabla F_m^i\left(\mathbf{x}_{t-1}^i\right)\right)\parallel^2\right]\\
    &\le \frac{1}{N}\mathbb{E}\left[\sum_{i=1}^{N}\parallel \nabla F_m\left(\bar{\mathbf{x}}_{t-1}\right)-\nabla F_m^i\left(\mathbf{x}_{t-1}^i\right)\parallel^2\right] \\
    &\le \frac{L^2}{N}\sum_{i=1}^{N}\mathbb{E}\left[\parallel\bar{\mathbf{x}}_{t-1}-{\mathbf{x}}_{t-1}^i\parallel^2\right] \\
    &\le 4\gamma^2I^2G^2L^2.
\end{aligned}\label{eq:s6-4}
\end{equation}

Since $\gamma \leq 1/L$, substituting equations (\ref{eq:s6-2}), (\ref{eq:s6-3}), and (\ref{eq:s6-4}) into (\ref{eq:s6-1}) yields
\begin{align}
&\mathbb{E}\left[F_m\left(\bar{\mathbf{x}}_t\right)-F_m\left(\bar{\mathbf{x}}_{t-1}\right)\right]
\le
2\gamma^3I^2G^2L^2+\frac{\gamma^2\sigma^2L }{2N}.\label{eq:s6-5}
\end{align}

For all $t$, we consider the least $t_c$ such that $t_c$ is a communication iteration and $t_c \ge t$ and have 
\begin{equation}
\begin{aligned}
\mathbb{E}\left[F_m\left(\bar{\mathbf{x}}_{t_c}\right)-F_m\left(\bar{\mathbf{x}}_{t}\right)\right]
&=\sum_{\tau=t}^{t_c-1}\mathbb{E}\left[F_m\left(\bar{\mathbf{x}}_{\tau+1}\right)-F_m\left(\bar{\mathbf{x}}_{\tau}\right)\right]\\
&\le
2\gamma^3I^3G^2L^2+\frac{\gamma^2\sigma^2IL }{2N}.
\end{aligned}\label{eq:s6-6}
\end{equation}

We reuse $t_c^m$ to denote the start of block $m$ in cycle $c$. Then, we define a sequence $\{\bar{\mathbf{x}}_t'\}$, where $\bar{\mathbf{x}}_t'=\bar{\mathbf{x}}_{t_c^m}$. We have 
\begin{equation}
\begin{aligned}
&\frac{1}{\parallel \mathcal{R}_c^m \parallel}\sum_{t_r\in  \mathcal{R}_c^m}\mathbb{E}\left[F_m\left(\bar{\mathbf{x}}_{t_r}\right)\right]-\frac{1}{K}\sum_{t=t_c^m}^{t_c^m+K-1}\mathbb{E}\left[F_m\left(\bar{\mathbf{x}}_{t}\right)\right]\\
=\ &\frac{1}{K}\sum_{t=t_c^m}^{t_c^m+K-1}\mathbb{E}\left[F_m\left(\bar{\mathbf{x}}_{t}' \right)-F_m\left(\bar{\mathbf{x}}_{t}\right)\right]
\le
2\gamma^3I^3G^2L^2+\frac{\gamma^2\sigma^2IL }{2N} ,
\end{aligned}\label{eq:s6-7}
\end{equation}
where $\mathcal{R}_c^m$ is the set of communication iterations in block $m$ from cycle $c$.

Then we bound the gap between the average loss of communication iterations and that of all the iterations:
\begin{equation}
\begin{aligned}
\frac{1}{\parallel \mathcal{R}\parallel}\sum_{t_r \in \mathcal{R}}F_m(\bar{\mathbf{x}}_{t_r}) - \frac{1}{T}\sum_{t=1}^T F_m\left(\bar{\mathbf{x}}_t\right)
&= \frac{1}{CM}\sum_{c,m}\left[\frac{1}{\parallel \mathcal{R}_c^m \parallel}\sum_{t_r\in  \mathcal{R}_c^m}\mathbb{E}\left[F_m\left(\bar{\mathbf{x}}_{t_r}\right)\right]-\frac{1}{K}\sum_{t=t_c^m}^{t_c^m+K-1}\mathbb{E}\left[F_m\left(\bar{\mathbf{x}}_{t}\right)\right]\right]\\
&\le 2\gamma^3I^3G^2L^2+\frac{\gamma^2\sigma^2IL }{2N}.
\end{aligned}\label{eq:s6-8}
\end{equation}
\end{proof}

\section{Proof of Theorem \ref{item:th2}} \label{th2:proof}
\subsection{Proof Sketch} \label{th2-sketch}
\begin{proof}[\textbf{Proof Sketch for Theorem \ref{item:th2}}]
We first separately bound the mini-max optimal errors for the models from two chains.
\begin{lemma}\label{lemma:th2-1}
By choosing $\gamma=\frac{\sqrt{N}}{L\sqrt{T}}$, $\eta=\frac{\sqrt{NM}}{L\sqrt{T}}$, and $I\le \frac{\left(T/M\right)^\frac{1}{4}}{N^{\frac{3}{4}}}$, when $K>CMN$, for $\left(M+1\right)$ chains, we have the following guarantee for the model from each chain:

The block-mixed model can guarantee
\begin{align}
\frac{1}{\parallel \mathcal{R}\parallel}\sum_{t_r\in \mathcal{R}}\mathbb{E} \left[F_m\left(\bar{\mathbf{x}}_{t_r}\right)\right]-F\left(\mathbf{x}^*\right)\le O\left(\frac{1}{\sqrt{NT}}\right).
\end{align}
For any block $m$, the block-separate model can guarantee
\begin{align}
\frac{1}{\parallel \mathcal{R}^m\parallel} \sum_{t_r\in \mathcal{R}^m} \mathbb{E}\left[F_m\left(\bar{\mathbf{y}}_{t_r}\right)\right]-F_m\left(\mathbf{x}_m^*\right)\le O\left(\frac{\sqrt{M}}{\sqrt{NT}}\right).
\end{align}
\end{lemma}
The result for the block-mixed chain directly comes from Theorem \ref{item:th1}. We next show the average loss of the chosen model is not more than that of the model from either chain.
\begin{lemma}\label{lemma:th2-2}
Based on the model selection principle for each iteration in MC-PSGD, we have
\begin{align}
&\frac{1}{\parallel \mathcal{R}\parallel}\sum_{t_r \in \mathcal{R}}\mathbb{E} \left[F_m\left(\bar{\mathbf{u}}_{t_r}\right)\right] \le \frac{1}{\parallel \mathcal{R}\parallel}\sum_{t_r\in R}\mathbb{E}\left[F_m\left(\bar{\mathbf{x}}_{t_r}\right)\right].
\end{align}
In addition,  
\begin{align}
&\frac{1}{\parallel \mathcal{R}^m\parallel}\sum_{t_r \in \mathcal{R}^m}\mathbb{E}\left[F_m\left(\bar{\mathbf{u}}_{t_r}\right)\right]\le \frac{1}{\parallel \mathcal{R}^m\parallel}\sum_{t_r \in \mathcal{R}^m}\mathbb{E}\left[F_m\left(\bar{\mathbf{y}}_{t_r}\right)\right],  \ \ \ \forall m \in \{1,2 \cdots, M \}.
\end{align}
\end{lemma}
By the above lemmas, we finally bound the average loss of the final predictors.
\begin{equation}
\begin{aligned}
\frac{1}{M}\sum_{m=1}^{M}\mathbb{E}\left[F_m\left(\tilde{\mathbf{u}}_m\right)\right]
&\le \frac{1}{M}\sum_{m=1}^{M}\frac{1}{\parallel \mathcal{R}^m \parallel}\sum_{t_r \in \mathcal{R}^m}\mathbb{E}\left[F_m(\bar{\mathbf{u}}_{t_r})\right]\\
&\le \frac{1}{\parallel \mathcal{R}\parallel}\sum_{t_r \in \mathcal{R}}\mathbb{E}\left[F_m\left(\bar{\mathbf{x}}_{t_r}\right)\right]
\le F\left(\mathbf{x}^*\right) + O\left(\frac{1}{\sqrt{NT}}\right).
\end{aligned}
\end{equation}
For any block $m$, we have
\begin{align}
\mathbb{E}\left[F_m\left(\tilde{\mathbf{u}}_m\right)\right]
\le \frac{1}{\parallel \mathcal{R}^m\parallel}\sum_{t_r \in \mathcal{R}^m}\mathbb{E}\left[F_m\left(\bar{\mathbf{u}}_{t_r}\right)\right]
\le \frac{1}{\parallel \mathcal{R}^m\parallel}\sum_{t_r \in \mathcal{R}^m}\mathbb{E}\left[F_m\left(\bar{\mathbf{y}}_{t_r}\right)\right]
\le F_m\left(\mathbf{x}_m^*\right) + O\left(\frac{\sqrt{M}}{\sqrt{NT}}\right).
\end{align}
\end{proof}

\subsection{Proof of Lemmas in Theorem \ref{item:th2}}\label{th2-lemma}
\begin{proof}[\textbf{Proof of Lemma \ref{lemma:th2-1}}]
According to Theorem \ref{item:th1}, if we choose $\gamma=\frac{\sqrt{N}}{L\sqrt{T}}, \eta=\frac{\sqrt{NM}}{L\sqrt{T}}$, and $I\le \frac{(T/M)^{\frac{1}{4}}}{N^{\frac{3}{4}}}$, when $K>CMN$, we have
\begin{align}
\frac{1}{\parallel \mathcal{R}\parallel}\sum_{t_r\in \mathcal{R}}\mathbb{E} \left[F_m\left(\bar{\mathbf{x}}_{t_r}\right)\right]-F\left(\mathbf{x}^*\right)\le O\left(\frac{1}{\sqrt{NT}}\right). \label{eq:s7-1}
\end{align}

And for any block-separate chain, it is equal to the non-cyclic case with $T/M$ iterations. According to \cite{DBLP:conf/aaai/Yu19}, we have
\begin{align}
\frac{M}{T}\sum_{t\in \mathcal{I}_m}\mathbb{E}\left[F_m\left(\bar{\mathbf{y}}_t\right)\right]-F_{m}\left(\mathbf{x}_m^*\right)
\le
O\left(\frac{\sqrt{M}}{\sqrt{NT}}\right), \label{eq:s7-2}
\end{align}
where $\mathcal{I}_m$ denotes the set of all the iterations corresponding to block $m$.

As for $\bar{\mathbf{y}}_t$, applying Lemma \ref{lemma:th1-6} to the block-separate training process for the block $m$ yields
\begin{align}
\frac{1}{\parallel \mathcal{R}^m\parallel}\sum_{t_r \in \mathcal{R}^m}\mathbb{E}\left[F_m(\bar{\mathbf{y}}_{t_r})\right]
\le \frac{M}{T}\sum_{t\in \mathcal{I}_m}\mathbb{E} \left[F_m\left(\bar{\mathbf{y}}_t\right)\right]+O\left(\frac{\sqrt{M}}{\sqrt{NT}}\right). \label{eq:s7-3}
\end{align}
Substituting equation (\ref{eq:s7-3}) into (\ref{eq:s7-2}) yields
\begin{align}
\frac{1}{\parallel \mathcal{R}^m\parallel}\sum_{t_r \in \mathcal{R}^m}\mathbb{E}\left[F_m\left(\bar{\mathbf{y}}_{t_r}\right)\right]-F_m\left(\mathbf{x}_m^*\right)\le O\left(\frac{\sqrt{M}}{\sqrt{NT}}\right). \label{eq:s7-4}
\end{align}
\end{proof}

\begin{proof}[\textbf{Proof of Lemma \ref{lemma:th2-2}}]
According to MC-PSGD, the server will choose a model with a smaller average local loss at each communication iteration $t_r$, so we have
\begin{align}
\frac{1}{N}\sum_{i=1}^N f\left(\bar{\mathbf{u}}_{t_r},\xi_{t_r}^i\right)
\le
\min\left\{\frac{1}{N}\sum_{i=1}^N f\left(\bar{\mathbf{x}}_{t_r},\xi_{t_r}^i\right),\frac{1}{N}\sum_{i=1}^N f\left(\bar{\mathbf{y}}_{t_r},\xi_{t_r}^i\right)\right\}, \label{eq:s8-1}
\end{align}
\end{proof}
where $\xi_{t_r}^i$ denotes the local data of client $i$ at communication iteration $t_r$.

Taking expectations on both side w.r.t. $\xi_{t_r}^i$ yields
\begin{align}
\forall t:\ 
\mathbb{E} \left[F_m\left(\bar{\mathbf{u}}_{t_r}\right)\right] 
\le 
\min\left\{
\mathbb{E}\left[F_m\left(\bar{\mathbf{x}}_{t_r}\right)\right],
\mathbb{E}\left[F_m\left(\bar{\mathbf{y}}_{t_r}\right)\right]
\right\}. \label{eq:s8-2}
\end{align}
So we have 
\begin{align}
&\frac{1}{\parallel \mathcal{R}\parallel}\sum_{t_r \in \mathcal{R}}\mathbb{E} \left[F_m\left(\bar{\mathbf{u}}_{t_r}\right)\right] \le \frac{1}{\parallel \mathcal{R}\parallel}\sum_{t_r\in R}\mathbb{E}\left[F_m\left(\bar{\mathbf{x}}_{t_r}\right)\right], \label{eq:s8-3}
\end{align}
and
\begin{align}
&\frac{1}{\parallel \mathcal{R}^m\parallel}\sum_{t_r \in \mathcal{R}^m}\mathbb{E}\left[F_m\left(\bar{\mathbf{u}}_{t_r}\right)\right]\le \frac{1}{\parallel \mathcal{R}^m\parallel}\sum_{t_r \in \mathcal{R}^m}\mathbb{E}\left[F_m\left(\bar{\mathbf{y}}_{t_r}\right)\right],  \ \ \ \forall m \in \{1,2 \cdots, M \}.
\end{align}
\section{Other Supplementary Lemmas}
\begin{lemma}\label{lemma:s1-1}
Under Assumptions \ref{assum:no2} and  \ref{assum:no5}, for  $\forall \mathbf{x},\mathbf{y},m$, we have
\begin{align}
F_m\left(\mathbf{x}\right)-F_m\left(\mathbf{y}\right)\le 2B^2L^2 + 2B^2 + 2B^2L.
\end{align}
\begin{proof}
By the $L$-smoothness of $F_m(\mathbf{x})$, we have
\begin{align}
F_m\left(\mathbf{x}\right)-F_m\left(\mathbf{y}\right)\le\left<\nabla F_m\left({\bf y}\right),\mathbf{x}-\mathbf{y}\right>+\frac{L}{2}\parallel \mathbf{x}-\mathbf{y} \parallel^2. \label{eq:smooth}
\end{align}
By the basic inequality and the bound on the model's norm, we have
\begin{equation}
\begin{aligned}
\left<\nabla F_m\left(\mathbf{y}\right),\mathbf{x}-\mathbf{y}\right>
&\le \frac{\parallel \nabla F_m\left(\mathbf{y}\right)\parallel^2+\parallel \mathbf{x}-\mathbf{y} \parallel ^2}{2}\\
&\le \frac{\parallel \nabla F_m\left(\mathbf{y}\right)\parallel^2+2\parallel \mathbf{x}\parallel ^2 +2\parallel \mathbf{y} \parallel ^2}{2}
\le \frac{\parallel \nabla F_m\left(\mathbf{y}\right)\parallel^2+4B^2}{2}. 
\end{aligned}\label{eq:deltaF}
\end{equation}
By the $L$-smoothness of $F_m(\mathbf{x})$, we have
\begin{align}
\parallel \nabla F_m\left(\mathbf{y}\right)\parallel^2
= \parallel \nabla F_m\left(\mathbf{y}\right) - \nabla F_m\left(\mathbf{x}_m^*\right) \parallel^2
\le L^2 \parallel {\bf y} - {\bf x}_m^* \parallel^2
\le 4B^2L^2. \label{eq:deltabound}
\end{align}
Substituting equation (\ref{eq:deltabound}) into (\ref{eq:deltaF}) yields
\begin{align}
\left<\nabla F_m\left(\mathbf{y}\right),\mathbf{x}-\mathbf{y}\right>
\le 2B^2L^2 + 2B^2. \label{eq:l9}
\end{align}
Substituting equation (\ref{eq:l9}) into (\ref{eq:smooth}) yields
\begin{align}
F_m\left(\mathbf{x}\right)-F_m\left(\mathbf{y}\right) 
\le
2B^2L^2 + 2B^2 + \frac{L}{2}\left(2 \parallel \mathbf{x}\parallel ^2 + 2\parallel \mathbf{y}\parallel^2\right)
\le 
2B^2L^2 + 2B^2 + 2B^2L. \label{eq:Fdiff}
\end{align}
\end{proof}
\end{lemma}
\end{document}